\documentclass[11pt]{article}
\usepackage[labelfont=bf]{caption}
\usepackage[utf8]{inputenc}
\usepackage{fullpage}
\usepackage[T1]{fontenc}
\usepackage{natbib}
\usepackage[colorlinks=true,citecolor=blue]{hyperref} 
\usepackage{url}            
\usepackage{booktabs}       
\usepackage{amsfonts}       
\usepackage{nicefrac}       
\usepackage{microtype}   
\usepackage{hhline}
\usepackage{makecell}

\usepackage{comment}
\usepackage{lipsum}
\usepackage{setspace}
\usepackage{mathtools}
\usepackage{graphicx}
\usepackage{amssymb}
\usepackage{bbm}
\usepackage{amsthm}
\usepackage{xpatch}
\usepackage{mathrsfs}

\usepackage{url}
\usepackage{array}
\usepackage{wrapfig}
\usepackage{tabularx}
\usepackage[normalem]{ulem} 
\usepackage{enumerate}
\usepackage{enumitem}
\usepackage[usenames,dvipsnames]{xcolor}
\usepackage{mathtools}

\usepackage{booktabs}       

\usepackage{footnote}

\usepackage{thmtools}
\usepackage{thm-restate}

\usepackage{hyperref}
\usepackage{cleveref}

\newtheorem*{thm*}{Theorem}

\newcommand{\Rc}{\mathcal{R}}
\newcommand{\Prob}{\mathbb{P}}
\newcommand{\E}{\mathbb{E}}

\usepackage{algorithm}
\usepackage{algpseudocode}

\title{Learning to Stop with Surprisingly Few Samples}

\author{Daniel Russo \\Columbia University \\ \texttt{djr2174@gsb.columbia.edu}
	\and
	Assaf Zeevi \\ Columbia University \\ \texttt{assaf@gsb.columbia.edu}
	\and
	Tianyi Zhang \\Columbia University \\ \texttt{tz2376@gsb.columbia.edu}
	}
\begin{document}

\maketitle

\begin{abstract}%
We consider a discounted infinite horizon optimal stopping problem. If the underlying distribution is known a priori, the solution of this problem is obtained via dynamic programming (DP) and is given by a well known threshold rule. When information on this distribution is lacking, a natural (though naive) approach is  ``explore-then-exploit," whereby the unknown distribution or its parameters are estimated over an initial exploration phase, and this estimate is then used in the DP to determine actions over the residual exploitation phase. We show: (i) with proper tuning, this approach leads to performance comparable to the full information DP solution; and (ii) despite common wisdom on the sensitivity of such ``plug in" approaches in DP due to propagation of  estimation errors,   a surprisingly ``short" (logarithmic in the horizon) exploration horizon suffices to obtain said performance.  In cases where the underlying distribution is heavy-tailed, these observations are even more pronounced: a {\it single sample}  exploration phase suffices.     
\end{abstract}

\section{Introduction}
\label{introsec}

\paragraph{The optimal stopping problem.} The folklore of  optimal stopping problems traces back to the work of the British mathematician Arthur Cayley in the late 1900's,  with the first rigorous formulation appearing in
\cite{moser1956problem}; for a review and historical account cf. \cite{ferguson1989solved}. The common structure in most optimal stopping problems considers a sequence $X_1,\ldots, X_n$
of independent random variables that, in the simplest instance, are drawn from a common distribution $F$. The  values of these random variables are   revealed sequentially and the
 player's objective, roughly speaking,  is to {\it stop} the sequence at a point where it is perceived to have reached its maximum value, and collect that as a reward. In this paper, we consider a discounted formulation of this problem (for reasons that will be discussed shortly): fixing a discount factor $\gamma \in (0,1)$, the player seeks to solve 
\begin{equation}
\label{problem1.1}
    V^* = \sup_{1 \leq \tau \leq n}{\mathbb{E}_{F}\left[\gamma^{\tau}X_{\tau}\right]},
\end{equation}
where the supremum is taken over all stopping times $\tau$ with respect to the sequence of observations that are bounded by the horizon length $n$.
Further technical details are deferred to Section \ref{sec-2}. 

This optimal stopping problem can be solved by dynamic programming (DP). The optimal  policy is given by a sequence of  fixed thresholds. Specifically, $\tau^* = \inf{\{1\leq i \leq n\: : \; X_i \leq {S_{i}^{(n)}}\}}$ where the thresholds $S_{i}^{(n)} =: A_{n-1}$ are determined the  backward recursions: $A_0 = 0;\,\,
A_j = \gamma\cdot \mathbb{E}_F \left[ \max \left\{X, A_{j-1}\right\} \right],\,\,  j = 1,\ldots,n$; see, e.g., \cite{bertsekas1995dynamic} for further discussion and an elaboration on the solution of the Bellman equation. Here $X$ denotes a generic draw from the distribution $F$ and the expectation is index by $F$ to make clear that the solution is directly determined by this problem primitive. 

 The problem outlined above has an especially elegant solution in the infinite horizon setting, namely, in the  asymptotic regime where $n \to \infty$. The 
optimal stopping rule  takes the form of a \textit{stationary} threshold policy:
\begin{equation}
\label{policy1.2}
    \tau^* = \inf{\{i\geq 1 \; :\; X_i \leq {S^*}\}},
\end{equation}
where the threshold is the unique solution of Bellman's fixed point  equation,
\begin{equation}
\label{threhsoldbellman}
    S^* = \gamma\cdot \mathbb{E}_F\left[\max{\{X, S^*\}}\right].
\end{equation}
It is due to this simplification that the infinite horizon  formulation is better suited for  highlighting  salient features of the learning problems described next.

\paragraph{The stopping problem under incomplete information.}
 A voluminous literature studies the optimal stopping problem and various variants thereof under complete information on the underlying primitives, primarily, the distribution $F$.  In contrast, and outside of work on so called secretary problems where typically one only observes relative ranks of $X_1,\ldots,X_n$ (cf. \cite{ferguson1989solved}),  very few antecedents consider the impact of incomplete prior information on achievable performance.  A brief review of such work is deferred to the end of this section; the reader is also referred to the recent paper by   \cite{goldenshluger2017optimal} and references therein. In particular,  the latter paper  considers a finite horizon problem, and for  nonparametric classes of distributions $F$, proposes a rank-based policy which is proven to be asymptotically  optimal (as the horizon grows large) relative to a benchmark given by the complete information solution.  Their proposed policy, which is derived as a solution to an auxiliary ranking problem,  possesses a relatively simple recursive structure but must be solved numerically.  
 
  An open question flagged in  \cite{goldenshluger2017optimal} is whether simpler families of policies  might  yield competitive  performance if the class of distributions $F$ is suitably  restricted. For example, natural algorithms might mimic the threshold structure of the optimal policy in \eqref{threhsoldbellman} by solving an auxiliary DP that replaces the true distribution $F$ with an estimate. A risk with such an approach is that even relatively small estimation errors could result in policies that perform poorly. In fact, it is widely recognized that solutions to Bellman's equation can be quite sensitive to perturbations in problem primitives when the discount factor is close to one; For example, see \cite{nilim2005robust} who discuss the potential propagation of errors in dynamic programming recursions.

\paragraph{Main contributions.} 
Under a reasonably broad parametric class of distributions, this paper gives a granular analysis of the optimal stopping DP and the impact of mis-estimation of the problem parameters on decision performance. This is used to establish the efficacy of simple policy which, in the spirit of model predictive control, uses some initial observations to estimate the distribution and subsequently optimizes performance assuming this ``plug-in" estimate were correct. We give a sharp asymptotic analysis as the discount factor $\gamma$ tends to 1, yielding several insights. 

The first insight pertains to the (minimal) number of samples needed to support learning in the optimal stopping problem under  incomplete  information.\footnote{With slight abuse of terminology, we will refer to this in what follows as {\it sample complexity}. As will become evident shortly,  the way we define this property  does not directly conform with  the  traditional $(\epsilon,\delta)$ PAC-learning interpretation, but it does capture the essence of the latter in our setting.} An exploration phase that collects only on the order of $\log(1/(1-\gamma))^2$ observations from $F$ suffices to learn enough about this distribution to make the ``plug in" approach near optimal. In contrast to the common wisdom described above that solutions to the Bellman equation can be quite sensitive even to small mis-estimations, collecting an amount of data that scales only logarithmically in the ``effective horizon'' is sufficient to support near optimal solutions in the case of optimal stopping.  
This threshold is sharp, in the sense that any number of samples which is of lower order is catastrophic for the decision-maker, losing half of the attainable value; see Theorem \ref{estimationthm}.  

Surprisingly, the length of the exploration horizon required to support near-optimal performance is lower in problems where the underlying distribution have tails that decrease more slowly. This is especially pronounced when the tails are heavy:  a single observation is revealing enough to make the ``plug-in'' policy asymptotically optimal; see Theorem \ref{estimationthm_onesample}.  This result is driven by an intrinsic robustness of decision quality to inaccurate estimation as the time horizon grows.

The results above are supported by a detailed analysis of the scale of parameter mis-estimation that can be tolerated without degrading decision performance. The critical scale is approximately $1/(\log (1/(1-\gamma)))$. In particular, any mis-estimation smaller than this scale induces a  threshold policy that is asymptotically optimal, and larger scale perturbations preclude that;  see Theorem \ref{thmpositive}. Moreover, the effects of said perturbation are highly asymmetric: the performance of the resulting threshold policy is far more sensitive to overestimation and it is reasonably robust to underestimation; see Theorem \ref{thmnegative}. This asymmetry is one of the key attributes that support our sample complexity findings.

\paragraph{Potential for broader relevance}
Our work focuses on a classic optimal stopping problem -- a problem which since its inception has served to illustrate features that pertain to  broader sequential decision making contexts. We use it to elucidate issues that arise specifically at the intersection of sequential decision making and statistical inference: the  sensitivity of DP solutions to mis-estimation and the sample complexity required to ensure the efficacy of plug-in type policies.  By revealing rich and unexpected behavior, the paper offers a detailed ``case-study" that motivates further investigation. We  are hopeful that this line of work will produce more general insights on the  sample complexity of learning in structured classes of dynamic optimization problems -- like inventory control or queuing control -- and how this depends on key problem primitives.

\paragraph{Related literature.} Our paper is connected to two major strands of literature. The first revolves around the optimal stopping problem, where there are only a few entries that study the incomplete information setting. The second is more directly related to the learning theory literature, for example, the efficacy of explore-then-commit (ETC) policies, but more broadly, reinforcement learning and sample complexity consideration in that space of problems.  (For brevity, we omit a review of the general principles of model predictive control and the like, which underlie the explore-then-commit and plug-in policies.)   These two strands are reviewed below in that order.  

The initial motivation to consider optimal stopping with partial information originates with the secretary problem.  \cite{stewart1978optimal} is among the first studies in this direction, and considers a uniform distribution $F$ with unknown mean and variance.  \cite{petruccelli1980best} considers an even more restricted setting where  only the mean is unknown. 
\cite{samuels1981minimax} considered minimizing the expected quantile of the observation selected based on Stewart’s work and constructs a minimax policy.  \cite{petruccelli1985maximin} considers normal distribution with unknown mean and variance, and  \cite{boshuizen1992moment} proposes a moment-based policy to handle a sequence of independent uniformly bounded random variables given that only the means and/or variances are known. All these antecedents restrict $F$ to have simple parametric form, and the proposed policies rely heavily on this supposition. This greatly restricts the breadth of insights one can tease out. In particular, neither one of these studies directly focus on the regret and sample complexity of the problem,  or their implications on implementation of simple yet universal learning rules like ETC. 

The more recent paper by \cite{goldenshluger2017optimal}, discussed earlier,  is an exception insofar as it studies broad nonparametric classes of distributions, and characterizes minimax regret  over said classes. The policy that is designed in that paper and shown to be minimax optimal (in order) relies on solving an auxiliary ranking problem. The latter is fairly intuitive yet leads to a far more complex learning algorithm compared to ETC-type policies. Moreover, that paper actually calls into  question the general prospects and efficacy of ``plug in" approaches in the context of optimal stopping with incomplete information.  As indicated earlier, our investigation is by and large triggered by this question and provides some initial (somewhat surprising) evidence to the contrary.    

  Our work also relates to broader research on reinforcement learning.  Despite a flurry of recent work in this area, very few papers develop understanding of specific structured decision problems, as ours does. Nearly all results on the sample complexity of online or offline reinforcement learning reveal that sample complexity scales super-linearly in the effective time horizon; see for example the recent works of \cite{azar2017minimax,jin2018q, agarwal2020model} and references therein. That is, learning to make near-optimal decisions requires more interactions than the time horizon allows.  Hence, successful RL algorithms learn across initial epochs of interaction how to optimize in future epochs. By contrast, our formulation requires optimizing within a single epoch and requires learning at a much faster timescale.  Results on learning in average cost Markov decision processes like those of \cite{jaksch2010near} appear to learn in a single episode, but require that the problem has small diameter, essentially meaning that any state is reachable in a small number of periods. It is unclear how such a notion could be adapted to optimal stopping: the state of the system (the most recent observation $X_n$) is continuous and unbounded, and the entire goal of the decision-maker is to reach states (i.e., very high values of $X_n)$ that are  just barely reachable within the given time horizon. 
  
  Two very recent papers suggest that it is possible to learn a near-optimal policy in general finite-state finite-horizon MDPs with a number of interactions that scale only logarithmically in the problem's time horizon \citep{wang2020long, zhang2020reinforcement}. It appears nontrivial to adapt their results, which depend polynomially on the number of distinct states, to our problem where the state space is continuous and unbounded. Nevertheless, when viewed in light of their work, our results suggest that the sample complexity of learning in a broad class of structured MDPs might also depend only logarithmically on the effective horizon. It is worth mentioning that our work provides very sharp asymptotic sample complexity results, including tight lower bounds, in a style that is quite different from \cite{wang2020long} and \cite{zhang2020reinforcement}.

\paragraph{Organization of the paper.} The paper is structured as follows. In Section 2 we provide the model, formulate the problem and  define the performance metric of relative regret. Section 3 describes the ``plug in" based explore-then-commit policy. Section 4 
contains the main results concerning the sample complexity requirements of the plug-in policy, including the result on sufficiency of a single observation. Section 5 provides the interpretation of the main results based on the probability of stopping and a non-traditional loss function, as well as a verify brief outline of the proofs. Some numerical results are presented in Section 6. Section 7 discusses the limitations and open questions. Proofs of all statements are given in the Appendix.

\section{Problem Formulation} \label{sec-2}
\paragraph{The model.} A player observes a sequence of independent and identically distributed random variables, $(X_1,X_2,\ldots)$ drawn from an unknown distribution $F$. At each stage $n=1,2,\ldots$ they may either:  stop the game and collect a reward given by $\gamma^n X_n$, where $\gamma \in (0,1)$ is a discount factor; or continue playing in which case the observed value is lost without recourse and the next observation in the sequence is revealed.  A common motivation for this set up is a  seller that wishes to maximize the expected net present value from the sale of a single indivisible item. The seller interacts with potential buyers across the  sequence of rounds and the random variables defined above are the offers s/he receives at each round. An {\it admissible policy} is given by a random stopping time $\tau$ adapted to the filtration $\left(\mathcal{F}_{n}\right)_{n \in \mathbb{N}}$ where $\mathcal{F}_n=\sigma(X_1, \cdots, X_n)$. The expected net present value earned by stopping rule $\tau$ under distribution $F$ can be written concisely as $\mathbb{E}_{F}\left[ \gamma^\tau X_{\tau} \right]$.  Optimizing this over admissible stopping times gives the value of the game, defined earlier in (\ref{problem1.1}).    
This will be viewed as the {\it full information benchmark}. Note that the player cannot solve this optimization problem directly as $F$, the key stochastic primitive, is not known a priori. 

\paragraph{Performance metric and objectives.} Motivated by the solution structure in the full information setting given in (\ref{policy1.2}) and (\ref{threhsoldbellman}), we seek to design a single stopping rule that offers competitive performance without knowledge of $F$. The shortfall in performance of an admissible policy induced by $\tau$ can be measured through its relative regret, defined by
\begin{equation}
	\label{relativeregret}
	\mathcal{R}(F, \gamma, \tau) := \frac{\mathbb{E}_{F}\left[\gamma^{\tau^*}X_{\tau^*}\right] - \mathbb{E}_{F}\left[\gamma^{\tau}X_{\tau}\right]}{\mathbb{E}_{F}\left[\gamma^{\tau^*}X_{\tau^*}\right]},
\end{equation}
where $\tau^*$ is the full information optimal stopping rule under $F$. Specifically, it takes the form of the threshold policy \eqref{policy1.2} whose threshold value is the solution to the Bellman equation \eqref{threhsoldbellman} under $F$. 

As is often the case in such problems, it will be instructive to consider the asymptotic behavior of the  relative regret, embedding it within a sequence of problems with discount factor tending to $1$. In effect, this means we are looking at increasingly longer ``effective time horizons."  Consider an algorithm that specifies a stopping rule $\tau_{\gamma}$ for a problem instance with discount factor $\gamma \in (0,1)$. We say this algorithm has \emph{universally vanishing relative regret on the distribution class $\mathcal{D}$} if  
\begin{equation}
\lim_{\gamma \to 1}  \mathcal{R}(F, \gamma, \tau_{\gamma}) =0  \quad \text{for every } F\in \mathcal{D}. 
\end{equation}
We usually omit the subscript on $\tau_{\gamma}$, as it is clear from context. Vanishing regret indicates that the price of only having partial knowledge that $F\in \mathcal{D}$ is negligible when the effective horizon is long. 

\paragraph{Admissible distributions and notation.} For tractability and to better elucidate key structural properties of the learning problem, we focus our analysis on a simple parametric family of distributions. We fix a scale parameter $\alpha \geq 1/2$ throughout the paper\footnote{The requirement that $\alpha \geq \frac{1}{2}$ is a technical condition due to the bounds in Lemma \ref{Bounds on Tailed Probability} in Appendix \ref{characterizingexp}. We conjecture  that Theorem \ref{estimationthm_onesample} holds for any $\alpha >0$.}.  For some location parameter $\theta \in \mathbb{R}$, we let  $F_{\theta}$ be the distribution whose probability density function $f_{\theta}$ takes the form
\begin{equation}
	\label{distribution}
	f_{\theta}(x) = C_0 \exp{\left(-|x-\theta|^{\alpha}\right)} 
\end{equation}
where $C_0$ is the normalization constant. We study whether particular algorithms have universally vanishing regret on the class distributions $\mathcal{D_\alpha} := \{F_{\theta}: \theta \in \mathbb{R}_{+}\}.$
Such algorithms must adapt to the unknown location parameter $\theta$. The value of $\alpha$ controls both the difficulty of estimating the location parameter from data and the likelihood of observing values much larger than the mean.  

 Some general notation that will be used in the sequel. For functions $f:(0,1)\to \mathbb{R}$ and $g:(0,1)\to \mathbb{R}$, we say $f(\gamma)=o(g(\gamma))$ if $f(\gamma)/g(\gamma) \to 0$ as $\gamma \to 1$ and $f(\gamma)=\omega( g(\gamma))$ if $f(\gamma)/ g(\gamma) \to \infty$ as $\gamma \to 1$. We write $g(\gamma) \sim f(\gamma)$ if $g(\gamma)/f(\gamma) \rightarrow 1$. To denote expectations taken over  $F_{\theta}$, we sometimes write $\mathbb{E}_{\theta}[\cdot]$ in place of the more cumbersome notation $\mathbb{E}_{F_{\theta}}[\cdot]$.The notation $\mathbb{P}_{\theta}[\cdot]$ is defined similarly. We write $\Rc(\theta, \gamma, \tau) := \Rc(F_{\theta}, \gamma, \tau)$ to denote relative regret for distributions in the location family $D_{\alpha}$.

\section{The Explore-then-Commit Plug-in Policy}
We study an extremely simple procedure for online learning in optimal stopping, presented in Algorithm \ref{alg:plugin}. The algorithm proceeds in three phases; for completeness its details are also summarized below in pseudo-code.  In the first, exploration phase,  it simply observes the first $N$ values,  and uses these  to learn about the unknown distribution. $N$ is a tuning parameter of the policy that will be expounded upon shortly.
In the second phase, a maximum likelihood estimator 
$\hat{\theta}$ is used to estimate the distribution's location parameter $\theta$. Finally, the algorithm "plugs in" said estimate into the Bellman equation, and optimizes it as if the estimated parameter value  were correct. Precisely, it employs the optimal threshold policy  described in \eqref{policy1.2} only using the plugged in estimate. The resulting stopping time can be written formally as 
\begin{equation}\label{tau-policy}
    \tau_N=\min\{n \geq N:  X_n \geq S^*(F_{\hat{\theta}} , \gamma ) \}.
\end{equation} 
We call this procedure the $N$ sample ``plug in"  explore-then-commit policy, or just ``plug in" policy for short. 

The approach described above is, as indicated, quite naive, and  there are certainly more refined learning algorithms than this family  of ETC solutions.  In particular, rather than terminating learning after the first $N$ stages, continuous improvement of the ``plug in" and the induced Bellman equation can be achieved via updating of said estimate also over the exploitation phase. Or one can  allow  the exploration phase to be determined in a fully  online manner as a function of collected observations rather being fixed in advance. However, as we will show next, the very simple ETC policy described above is remarkably effective as is, without these modifications.  The performance guarantees that spell this out add to the already attractive nature of the policy due to its apparent simplicity. This follows a vast literature, usually referred to as model predictive control,  that tackles dynamic optimization under incomplete information by ``separating"    statistical inference and decision-making.  The challenge is then to show when  such approaches lead to provably near optimal performance. This is what we establish in this paper in the context of our optimal stopping problem.

\begin{algorithm}\label{alg1}
\caption{Explore-then-commit policy}
\label{alg:plugin}
\textbf{Input:} Distribution's shape parameter $\alpha \geq 1/2$, discount factor $\gamma <1$, sample size $N \in \mathbb{N}^{+}$.
\begin{algorithmic}[1]
\State Observe and skip the first $N$ samples $X_1,\,...\, X_N$. Compute $\hat{\theta}  = \arg\max_{\theta} \sum_{i=1}^{n} \log f_{\theta}(X_i)$.
\State Solve for a threshold $S^*(F_{\hat{\theta}}, \gamma)$ satisfying the Bellman Equation \eqref{threhsoldbellman} when the underlying distribution is $F_{\hat{\theta}}$.
\State For each period $n>N$, observe the offer $X_n$ and stop if it exceeds the threshold $S^*(F_{\hat{\theta}}, \gamma)$.
\end{algorithmic}
\end{algorithm}

\section{Main Results: Sample Complexity}
\label{sec:main}
\paragraph{Light-tailed distributions.} Our main results concern the first order asymptotic behavior of the relative regret (\ref{relativeregret}) under the ETC $N$ sample plug-in policy, defined via the stopping time $\tau_N$ in (\ref{tau-policy}).  The first result demonstrates a sharp phase-transition with respect to the sample size. Specifically, the plug-in policy has universally vanishing relative regret on the class of distributions $\mathcal{D_\alpha}$ if its sample size exceeds the critical threshold identified below.  When the sample size falls below this threshold, relative regret tends to 1/2, meaning that the plug-in policy loses precisely half of the maximal value obtained in the full information problem.  Vanishing regret also  requires the fairly obvious requirement that the initial  sample size $N$ is $o(1/(1-\gamma))$, effectively meaning that the algorithm does not forego a constant fraction of the problem's time horizon on estimating the  distribution. For exploration sample sizes that grow more slowly than the critical threshold, attaining the sharp constant of 1/2 requires that $N=\omega(1)$ so that our asymptotic analysis applies. (Recall the notation $a_n = \omega(b_n)$ means $\lim_{n \to \infty} a_n/b_n =\infty $.)     Define the critical sample size scaling
\begin{equation}\label{eq:N_critical}
N_{\rm critical}(\gamma)=\left(\frac{\left(\log{\frac{1}{1-\gamma}}\right)^{1-\frac{1}{\alpha}}}{\log\log{\frac{1}{1-\gamma}}}\right)^2.    
\end{equation}


\begin{restatable}{thm}{estimationthm}\textbf{(Sample complexity light-tailed distributions)} 
\label{estimationthm}
Suppose $\alpha >1$.
\begin{itemize}
    \item[(i.)] If $N = o\left(N_{\rm critical}(\gamma)\right)$ and $N=\omega(1)$, then
    $$\lim_{\gamma \rightarrow 1} \Rc(\theta, \gamma, \tau_N) = \frac{1}{2}  \quad \text{for any } \theta \in \mathbb{R}.$$
    \item[(ii.)] If  $N = \omega\left(N_{\rm critical}(\gamma)\right)$ and $N = o\left(\frac{1}{1-\gamma}\right)$, then 
    $$\lim_{\gamma \rightarrow 1} \Rc(\theta, \gamma, \tau_N) = 0  \quad \text{for any } \theta \in \mathbb{R}.$$
\end{itemize}
\end{restatable}

\paragraph{Discussion.} Note that we always have the upper bound $N_{\rm critical}(\gamma) = O\left(\left(\log{\frac{1}{1-\gamma}}\right)^{2}\right)$, which is independent of $\alpha$. The term $1/{(1-\gamma)}$ can be interpreted as the length of effective time horizon. Thus Theorem \ref{estimationthm} demonstrates that the ETC policy relies on a plug-in that is \emph{nearly independent} of the problem's horizon. 

In contrast, when insufficient samples are collected, the relative regret tends to precisely 1/2, which is less than the maximal relative regret of 1 but is still completely independent of the distribution's location and scale parameters. We show in Section \ref{sec:analysis_ideas} that this is due to an intrinsic asymmetry in the estimation effects on the regret, where the plug-in policy is far more robust to underestimation than overestimation.

Perhaps the most surprising component of this result is that sample complexity requirements are milder when the distribution is heavier tailed. Specifically, notice that the distributions in \eqref{distribution} are heavier tailed for smaller values of $\alpha$ but that the critical sample size \eqref{eq:N_critical}  is decreasing in $\alpha$. Although estimation is more difficult for smaller $\alpha$, the performance of plug-in policies is less sensitive to estimation errors in such problems and this latter effect dominates. At the extreme, the critical sample size in Theorem \ref{estimationthm} seems to vanish as $\alpha$ tends to 1. 

\paragraph{Heavy-tailed setting.} The  next result confirms this observation on heavier-tailed settings, showing remarkably that for distributions with sub-exponential tails the \emph{single sample} plug-in policy has vanishing regret. (Recall that $\tau_1$ denotes the plug-in policy when $N=1$.) Of course, the single-sample plug-in estimator is inaccurate, especially when the distribution has heavy tails. This result is driven by the intrinsic robustness of decision quality to inaccurate estimation as the time horizon grows. 



\begin{restatable}[Sufficiency of a single observation]{thm}{onesample}
\label{estimationthm_onesample}
If $\alpha \in \left[\frac{1}{2},1 \right]$, then
    $$\lim_{\gamma \rightarrow 1} \mathcal{R}(\theta, \gamma, \tau_1) = 0 \quad \text{for any }  \theta \in \mathbb{R}.$$
\end{restatable}

We note that the results presented above focus on  pointwise convergence 
for the plug-in policy that hold under any fixed location parameter. A careful reading of the proof techniques in Appendix \ref{uniformapprox} shows that relative regret also 
converges uniformly on compact sets. That is, taking $N=1$ if 
$\alpha \in [1/2,1]$ and $N=\omega(N_{\rm critical}(\gamma))$ if $\alpha>1$,
it holds that for any real numbers $\theta_1 \leq \theta_2$, 
$$\lim_{\gamma \rightarrow 1} \sup_{\theta \in [\theta_1,\theta_2]} \Rc(\theta, \gamma, \tau_1) = 0.$$
In the next section we explore in more detail the underlying structural properties of the DP solution and the Bellman equation and their implications on required estimation accuracy.  This will be presented in the form of an infinitesimal perturbation and sensitivity analysis.

\section{Theoretical Foundations: Sensitivity and Perturbation Analysis of the Full Information Problem}\label{sec:analysis_ideas}
The statistics governing the plug-in policy are quite simple, as we know that  under general conditions the maximum likelihood estimator is asymptotically normal. The difficulty in our  analysis is to understand how estimation errors that are on the order of $1/\sqrt{N}$ impact downstream decisions made by the algorithm. To frame this question more precisely, for given  discount factor let $\tau^*(\theta)$ denote the optimal stopping rule were the distribution known to be $F_{\theta}$. Our analysis centers around understanding the relative regret $\mathcal{R}\left( \theta, \gamma, \tau^*(\theta + \epsilon)  \right)$ incurred by employing  the optimal policy $\tau^*(\theta+\epsilon)$ under a location parameter that is perturbed from the truth by some $\epsilon$. This $\epsilon$ can roughly be thought of as the error in estimating the true parameter $\theta$, but throughout the section we take a non-stochastic view and study the sensitivity to arbitrary perturbations. 

Define the critical perturbation magnitude to be 
\begin{equation}\label{eq:critical_epsilon}
	\epsilon_{\rm critical}(\gamma)= \frac{1}{\alpha \sqrt{N_{\rm critical}(\gamma)}}= \frac{\log{\log{\frac{1}{1-\gamma}}}}{\alpha \left(\log{\frac{1}{1-\gamma}}\right)^{1-\frac{1}{\alpha}}}.
\end{equation}
The next result shows that comparing against this critical scale  determines whether a plug-in policy will be robust to overestimation of the location parameter; the case of underestimation is treated separately for reasons that will become obvious shortly.  
\begin{restatable}[Phase-transition for relative regret under overestimation]{thm}{thmpositive}
	\label{thmpositive}
	If  $\epsilon : [0,1]\to \mathbb{R}_{+}$ satisfies $\liminf_{\gamma \to 1}  \epsilon(\gamma)/ \epsilon_{\rm critical}(\gamma) > 1$, then 
	$$\lim_{\gamma \rightarrow 1} \mathcal{R}(\theta, \gamma, \tau^*({\theta+\epsilon(\gamma)})) = 1  \quad \text{for any } \theta \in \mathbb{R}.$$
	If  $\epsilon : [0,1]\to \mathbb{R}_{+}$ satisfies   $\limsup_{\gamma \to 1}  \epsilon(\gamma)/ \epsilon_{\rm critical}(\gamma) < 1$, then
	$$\lim_{\gamma \rightarrow 1} \mathcal{R}( \theta, \gamma, \tau^*({\theta+\epsilon(\gamma)})) = 0 \quad \text{for any } \theta \in \mathbb{R}.$$
\end{restatable}

The next result shows that performance is quite robust to underestimation, revealing a fundamental asymmetry.  For any fixed scalar $\epsilon>0$, employing the plug-in policy with underestimated location parameter $\theta - \epsilon$ yields vanishing relative regret. Vanishing regret is even possible when $\epsilon$ grows ``slowly" with the discount factor. 
\begin{restatable}[Robustness to underestimation]{thm}{thmnegative}
	\label{thmnegative}
	If $\epsilon : [0,1]\to \mathbb{R}_{+}$ satisfies
	\begin{equation}
		\label{epsilonbounded}
		\epsilon(\gamma) = o\left(\left(\log{\frac{1}{1-\gamma}}\right)^{\frac{1}{\alpha}}\right) \quad \text{as } \gamma \to 1, 
	\end{equation}
	then 
	$$\lim_{\gamma \rightarrow 1} \mathcal{R}(\theta, \gamma, \tau^*({\theta-\epsilon(\gamma)})) = 0 \quad \text{for any } \theta\in \mathbb{R}.$$
\end{restatable}
Our main results from the previous section follow relatively easily from these two theorems. For instance, consider the claim in Theorem \ref{estimationthm} that $\lim_{\gamma \rightarrow 1} \Rc(\theta, \gamma, \tau_N) = \frac{1}{2}$ when $\alpha>1$ and $N$ grows slower than the critical sample size. This can be understood by imagining that the estimated location parameter used by the $N$ sample plug-in policy follows a $N(\theta, \sigma^2/N  )$ distribution. This is approximately the case, due to classical asymptotic theory of  maximum likelihood estimation. Now, there is roughly a 1/2 chance of underestimation, to which performance is robust. There is roughly a 1/2 chance of overestimation by a magnitude on the order of $\sigma \sqrt{N}$. By Theorem \ref{thmpositive}, this event leads to  relative regret of 1 when $N< N_{\rm critical}(\gamma)$. 

\paragraph{Interpretation based on the probability of stopping.} 
The next proposition helps clarify where this critical threshold (and hence the critical sample size) arises from. Notice that when the decision-maker underestimates the location parameter, they are willing to accept and stop on slightly lower values  than is optimal. As a result, they can be expected to earn slightly less, but stop earlier. The results above suggest performance is relatively robust to such underestimation even if the time horizon is extremely long. When the location parameter is overestimated, the decision-maker tends to  reject and not stop for values they would have accepted under an optimal policy. This may lead to a larger ultimate  value  $\mathbb{E}[X_{\tau}]$, but it could entail passing on an extremely sequence of observation before stopping. This risk is what drives the sensitivity to overestimation in Theorem \ref{thmpositive}. 

The next result interprets the critical perturbation magnitude \eqref{eq:critical_epsilon} in terms of the expected stopping time. It shows that when the decision-maker overestimates the location parameter by more than this critical magnitude, they are unlikely to stop within the problem's effective time horizon of $1/(1-\gamma)$. The decision-maker's mistaken optimism about future observations causes them to reject nearly all of them. Overestimation by less than the critical magnitude does not impact the decision-maker as severely and they still stop well within the effective time horizon. While Proposition \ref{prop:stopping_prob} appears to study the mean of the stopping time, it is worth noting that it effectively characterizes the full distribution because $\tau^*(\theta+\epsilon)$ has a geometric distribution.

\begin{restatable}[Phase transition for the stopping time]{prop}{propOnTimeStopping}
\label{prop:stopping_prob}
    \leavevmode\newline If $\epsilon : [0,1]\to \mathbb{R}_{+}$ satisfies $\liminf_{\gamma \to 1}  \epsilon(\gamma)/ \epsilon_{\rm critical}(\gamma) > 1$, then for every $\theta \in \mathbb{R}$,
	\[
	\mathbb{E}_{\theta}\left[ \tau^*(\theta+\epsilon(\gamma)) \right] = \omega\left( \frac{1}{1-\gamma} \right). 
	\]
	If $\epsilon : [0,1]\to \mathbb{R}_{+}$ satisfies $\limsup_{\gamma \to 1}  \epsilon(\gamma)/ \epsilon_{\rm critical}(\gamma) < 1$, then for every $\theta \in \mathbb{R}$,
	\[
	\mathbb{E}_{\theta}\left[ \tau^*(\theta+\epsilon(\gamma)) \right] = o\left( \frac{1}{1-\gamma} \right).  
	\]
\end{restatable}

\paragraph{Greater robustness with heavier tails.}
Paralleling our discussion in Section \ref{sec:main}, $\epsilon_{\rm critical}(\gamma)$ increases with $\alpha$, suggesting greater robustness for heavier tailed distributions.  In fact, for heavy-tailed distribution ($\alpha \leq 1$), we find $\epsilon_{\rm critical}(\gamma) \to \infty$ as $\gamma \to 1$, implying robustness to increasingly large estimation errors. This increasing robustness is what drives Theorem \ref{estimationthm_onesample}. 

A rough explanation  of this phenomenon is as follows.  Recall the form of the probability density function, $f_{\theta}(x-\theta) = C_0 {\rm exp}(-x^\alpha)$. Lemma \ref{newsandwich} shows that the optimal threshold behaves roughly as $S^*(\theta, \gamma) \approx \theta + \log\left( 1/[1-\gamma] \right)^{1/\alpha}$. A rough idea of this effect  arises  from plugging this approximation on the right hand side into the PDF, giving  $f_{\theta}( S^*(\theta, \gamma) -\theta ) \approx C_0 (1-\gamma)$. This  suggests that the optimal threshold can be found in the tail of distribution where  with density is  roughly proportional to $(1-\gamma)$. When the location parameter undergoes a positive  perturbation by a some small $\epsilon>0$, the threshold employed by the plug-in policy also increases by roughly $\epsilon$. Such a perturbation has a much more significant impact on light tailed distributions. Here the density $f_{\theta}(x)$ decays very rapidly around $x\approx S^*(\theta, \gamma)$, meaning that it may be \emph{extremely unlikely} to observe value that are $\epsilon$ larger than the optimal threshold. Heavier tailed distributions are flatter around their optimal threshold so the same magnitude of overestimation may not have as dramatic an effect. Flatter tails make performance less sensitive to miss-estimation of the stopping threshold.  The case where $\alpha<1$ is precisely when the PDF is strictly log-convex, so that the logarithm of the PDF becomes increasingly flat near the optimal threshold as the horizon grows.  


\paragraph{Loss function interpretation.} Studying the regret of the plug-in policy reduces, effectively, to studying the estimation error of the location parameter in a custom loss function that captures the downstream impact of  estimations errors on actions.   To describe this more carefully, write $\ell_{\gamma}(\theta, \hat{\theta}) = \mathcal{R}\left( \theta, \gamma,  \tau^*(\hat{\theta}) \right).$ A loss function like this has rich properties that are not seen in common instances like the squared or log loss. First, it is natural to consider $\gamma$ not as a fixed parameter, but as one that grows with the sample size, capturing how data requirements scale with the problem's effective time horizon. In the limit as $\gamma \to 1$, $\ell_{\gamma}(\theta, \hat{\theta})$ becomes incredibly asymmetric in its second argument: one can show that for small $\epsilon>0$, $\ell_{\gamma}(\theta, \theta -\epsilon) = O(\epsilon)$ as $\gamma \to 1$ while $\ell_{\gamma}(\theta, \theta +\epsilon) \to 1$ as $\gamma \to 1$. This characteristic, where a negligible but constant overestimation leads to a catastrophic loss for the decision-maker already suggests that the sample size must grow with the problem's time horizon. Of course, estimating the location parameter is more difficult when the distribution is heavy tailed ($\alpha <1$), but the loss function becomes less sensitive in its second argument in that case. This phenomenon highlights the interaction between decision-making (stopping)  and statistical inference; namely,  the loss function is derived endogenously and is not hypothesized upfront.

\paragraph{Outline of the proofs.}
We provide an overview of analysis and proofs of Theorem \ref{thmpositive} and Theorem \ref{thmnegative}. At a high level, the analysis of these optimal stopping problems as the discount factor tends to 1 revolves around analyzing tail behavior of the class of distributions. The challenging factor is that the optimal threshold policy $S^*(\theta, \gamma)$ is only defined implicitly as a solution to Bellman's fixed point equation underlying (\eqref{problem1.1}). To develop, for example, sharp asymptotics regarding the probability of stopping under a perturbed plug-in policy (see Prop~\ref{prop:stopping_prob}), requires studying the behavior of $1-F_{\theta}(S^*(\theta + \epsilon(\gamma), \gamma))$ as $\gamma \to 1$. This is the the inverse CDF evaluated at the plug-in stopping threshold $S^*(\theta + \epsilon(\gamma), \gamma))$ under the perturbation $\epsilon(\gamma)$. This in itself requires sharp asymptotics for the threshold $S^*(\theta + \epsilon(\gamma), \gamma))$ under growing discount factor and  growing perturbations. These asymptotic considerations are worked out rigorously  in the Appendix.

\section{Numerical Illustration}
\label{numericalsec}
In this section, we present a simple numerical illustration. Our presentation focuses on three common exponential family distributions that closely mimic the tail behavior of \eqref{distribution} with different choices of $\alpha$: the Weibull  distribution $(\alpha =1/2)$;  the exponential distribution ($\alpha=1$);  and the Normal distribution ($\alpha =2)$. Figure \ref{fig:1}, compares the expected discounted reward generated by the plug-in policies against an optimal policy that knows the distribution's mean. Under the heavier-tailed Weibull distribution, the plug-in policy that estimates the location parameter based on $N=1$ samples provides essentially perfect performance. Mirroring our theory, more samples are required to have comparatively strong performance under the exponential or normal distributions. For the normal experiment, the line labeled ``nearly-horizon independent number of samples'' refers to the sample size $N = \left(\log{\frac{1}{1-\gamma}}\right)^2$ that serves as an upper bound on the scaling of the critical samples size in \eqref{eq:N_critical}.

\begin{figure}
    \centering
    \includegraphics[width=5cm]{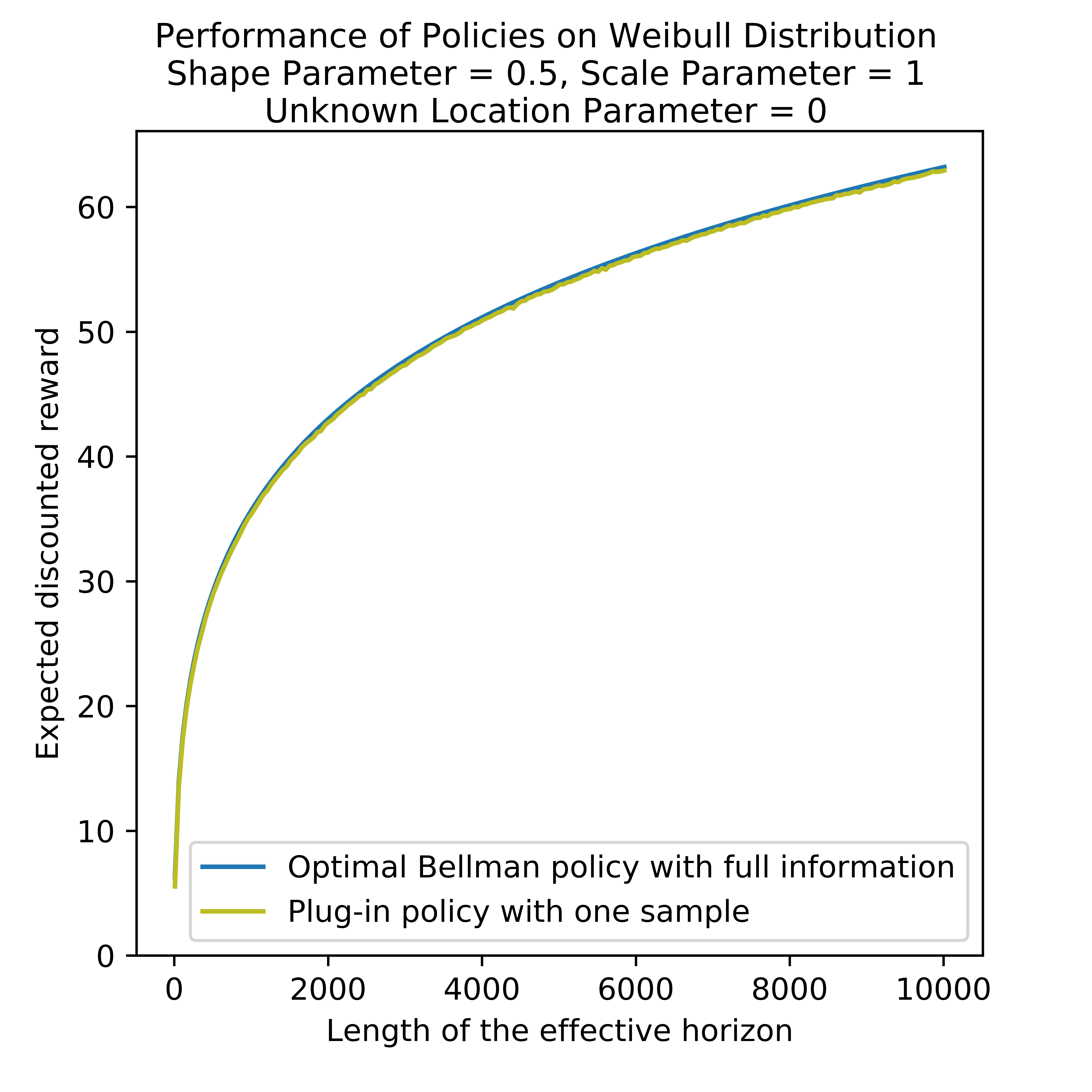}
    \includegraphics[width=5cm]{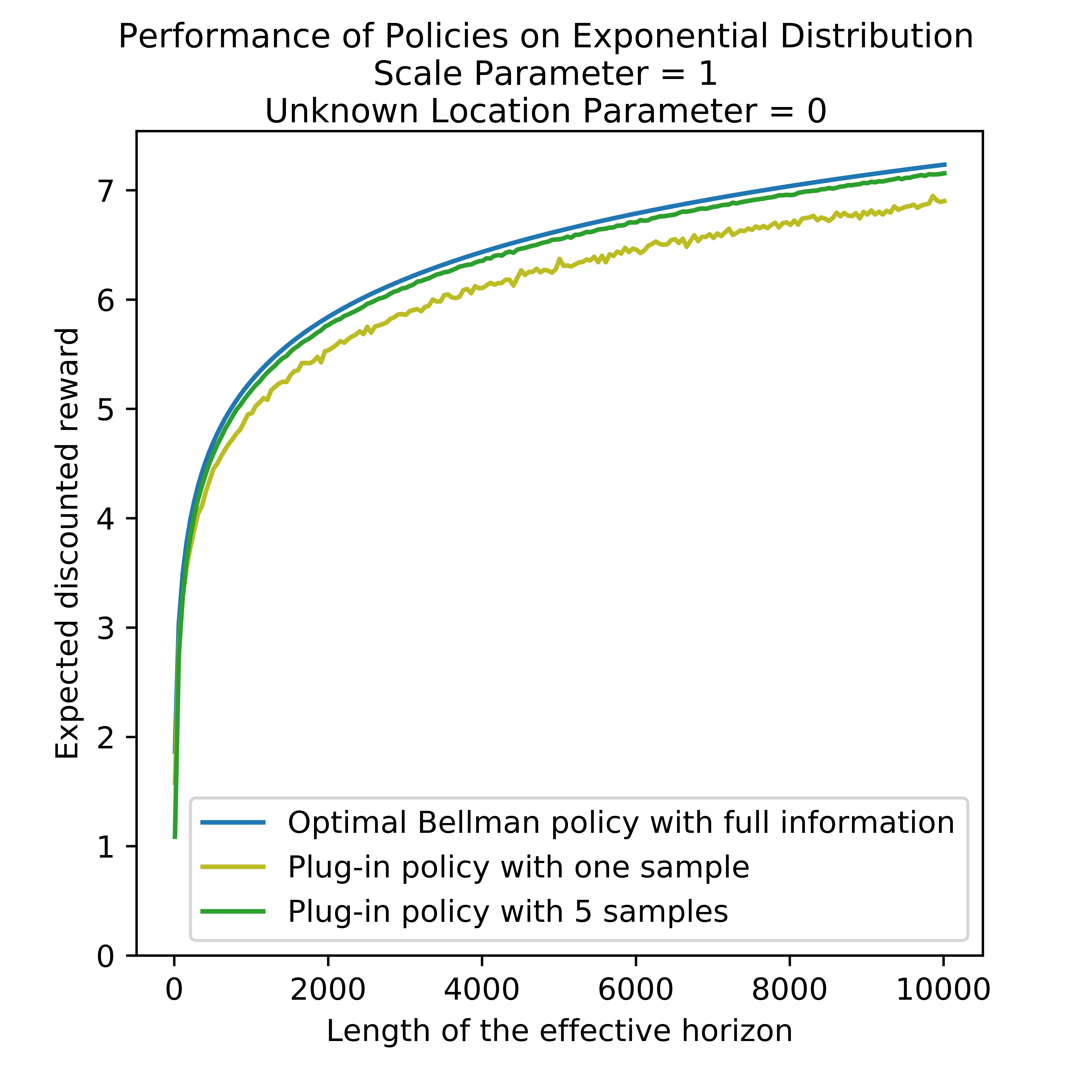}
    \includegraphics[width=5cm]{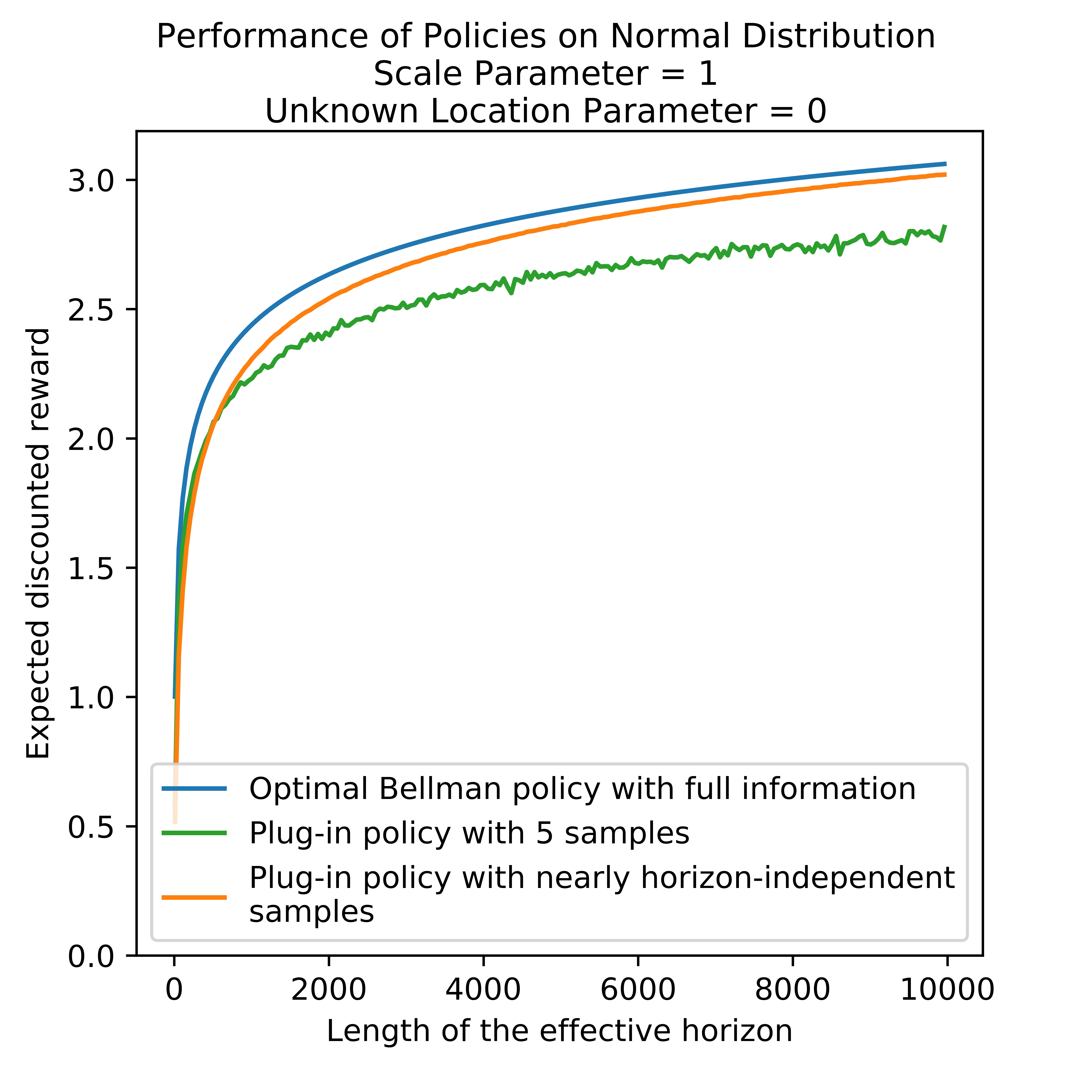}
    \caption{Performance of explore-then-plug-in policies across problems with varying discount factor $\gamma$. The horizontal axis displays the effective horizon $(1-\gamma)^{-1}$.  }
    \label{fig:1}
\end{figure}

Figure \ref{fig:2} studies the performance of a threshold policy $\tau^*(\theta + \epsilon)$ computed with respect to a location parameter that deviates from the truth of $\theta$ by some given perturbation level. Notice that different magnitudes of mis-estimation are reasonably likely to manifest under the Weibull, exponential, and normal distributions, respectively.  To allow for meaningful comparisons, we set $\epsilon=z\sigma$ where $\sigma$ is the standard deviation of the distribution and plot results as the multiplier $z$ varies. The standard error of the sample mean is proportional to $\sigma$ in each case. For the heavier-tailed Weibull distribution, performance is quite insensitive to perturbation of the location parameter, even as the horizon grows. For the normal, performance becomes increasingly sensitive to overestimation as the effective time horizon grows. Asymmetry in the curves also becomes more pronounced as the time horizon grows; slight overestimation results in significant  performance loss but underestimation, in contrast,  has only a mild impact. This plot seems to suggest that effective optimal stopping policies must carefully guard against mis-estimation. Given this intuition, it may be surprising that the naive plug-in policy produces competitive performance while collecting a nearly horizon independent number of samples.

\begin{figure}
    \centering
    \includegraphics[width=5cm]{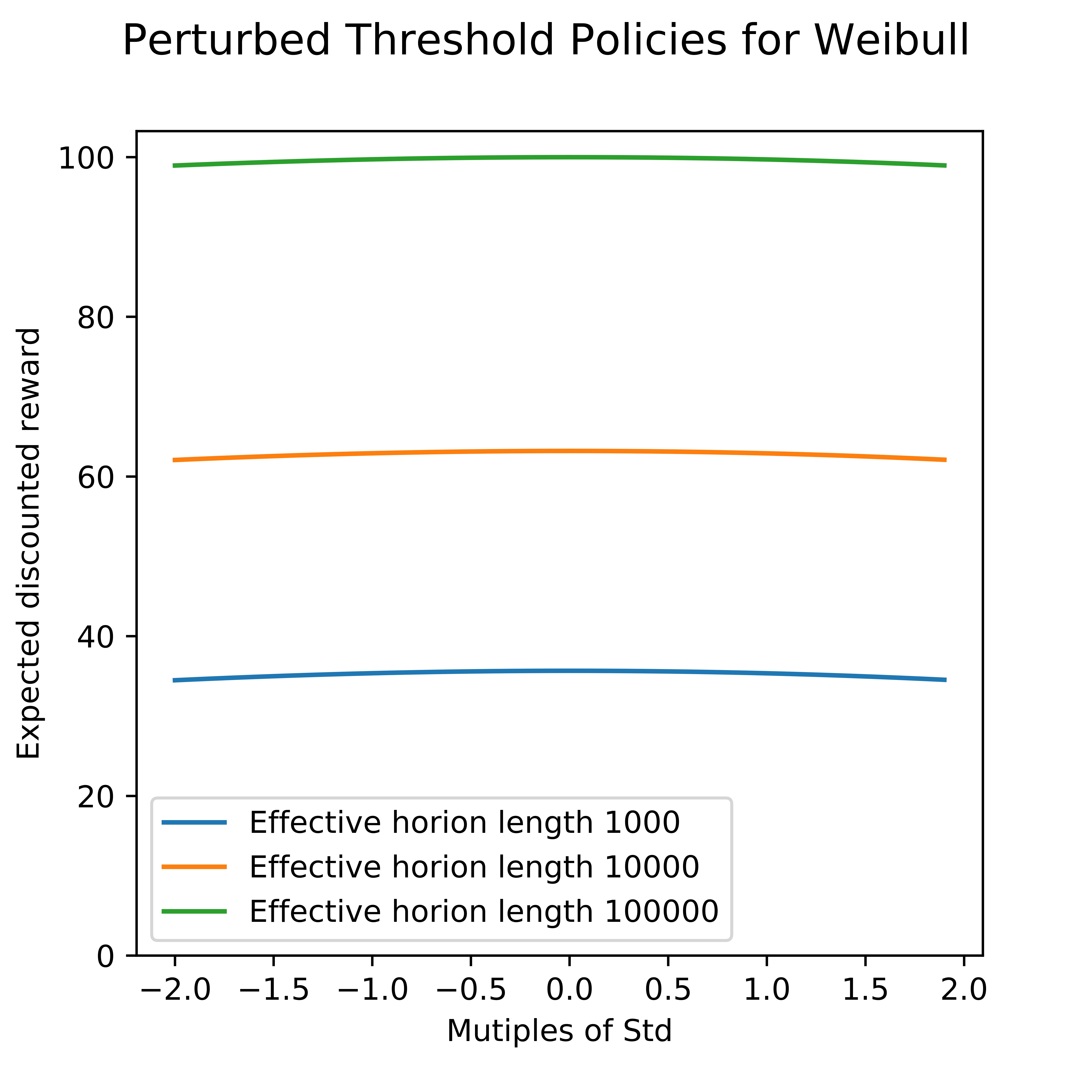}
    \includegraphics[width=5cm]{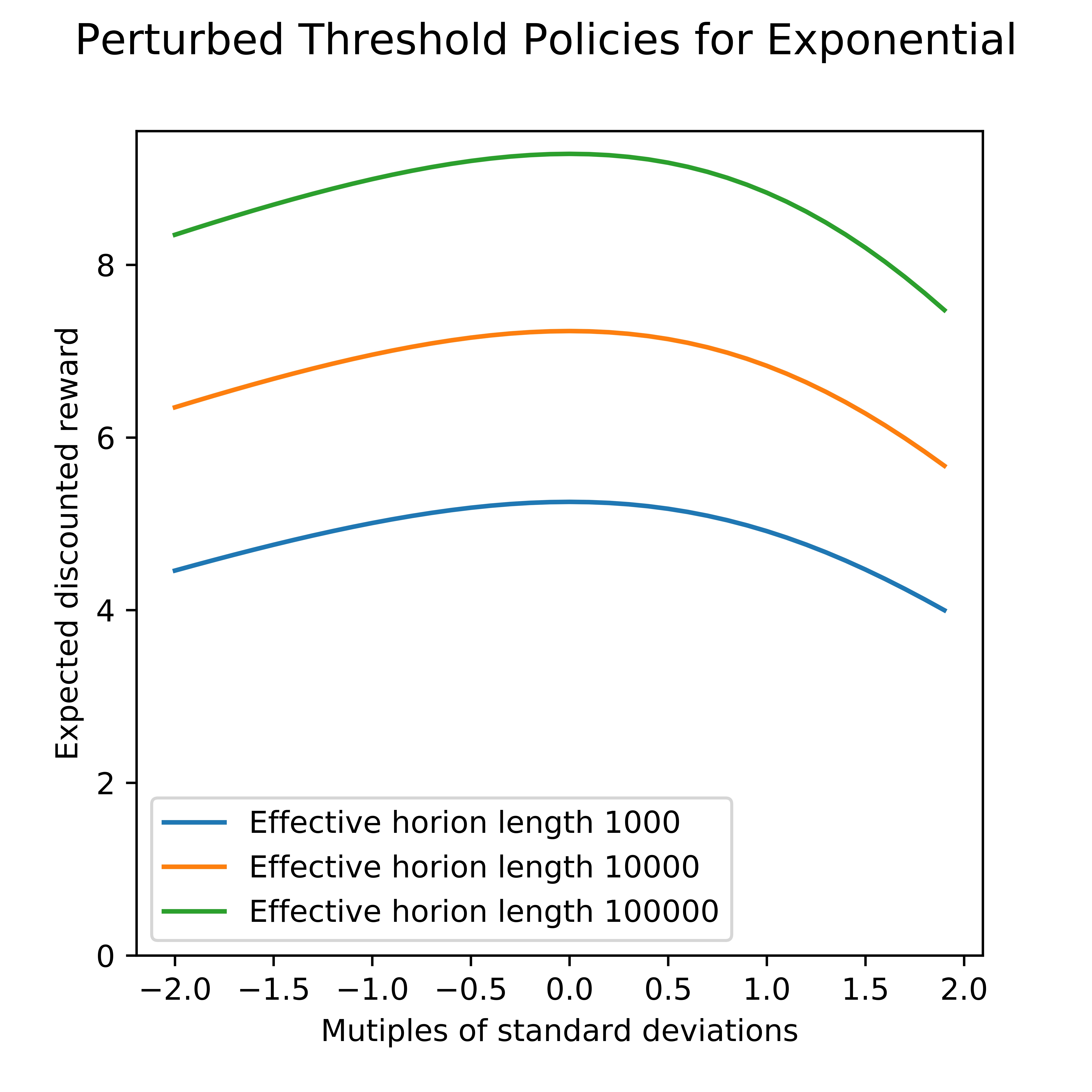}
    \includegraphics[width=5cm]{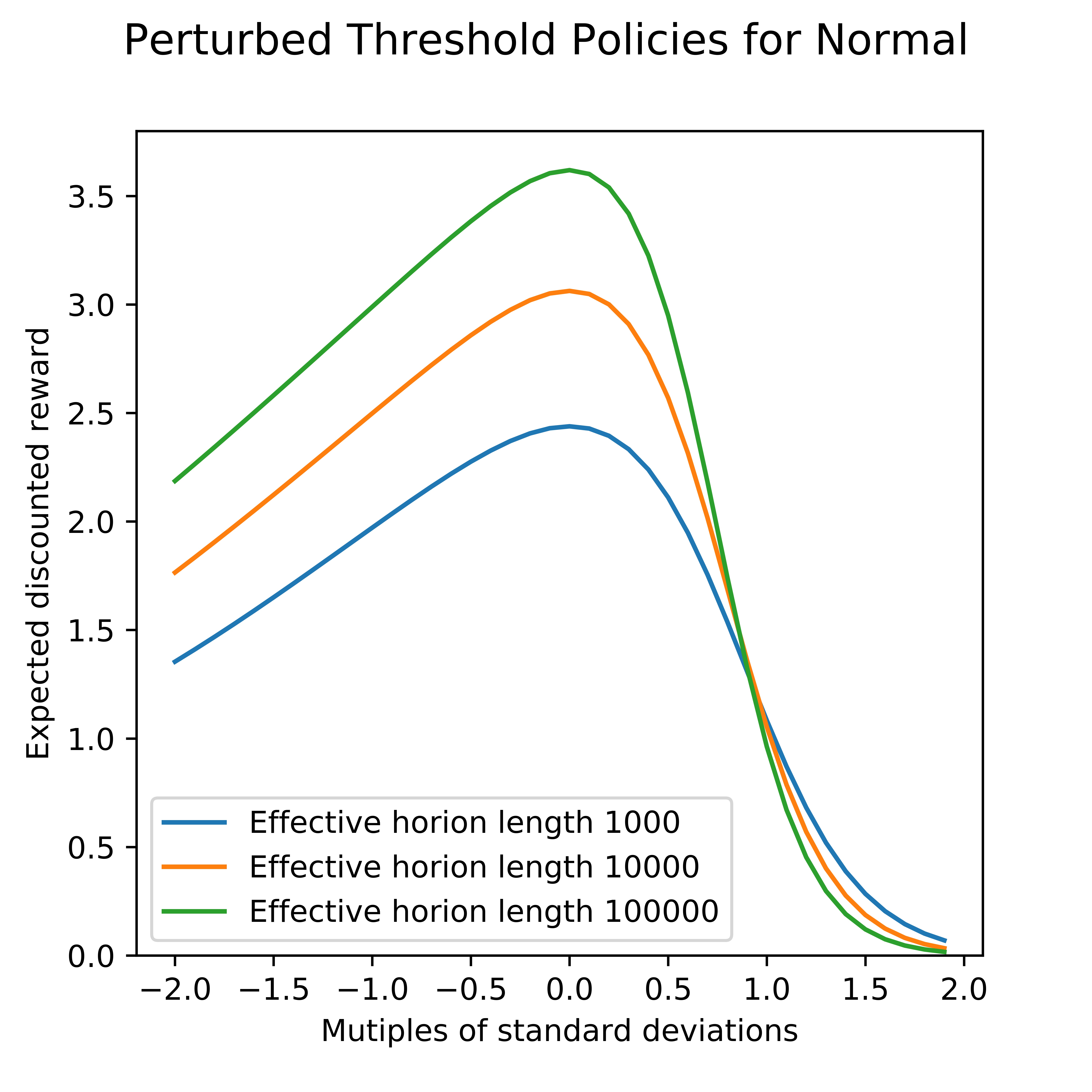}
    \caption{Performance of the threshold policy $\tau^*(\theta + z \sigma)$ computed with a location parameter $\theta$ that deviates from the truth a multiple $z$ times the standard deviation $\sigma$ of the offer distribution. }
    \label{fig:2}
\end{figure}

\section{Discussion and Open Problems}

Throughout the paper, our goal was  to state theoretical results that give a crisp illustrations of these findings, and by that optimizing the presentation to facilitate  key takeaways rather than striving for generality.
One limitation of this approach  is the rather specific choice of the parametric family of densities in \eqref{distribution} governing the optimal stopping problem.  While knowledge of the scale parameter was assumed ex ante, preliminary numerical experiments suggest that the plug-in policy performs well with very little data even when the scale parameter is unknown a priori. Confirming this rigorously is an open question. Our results may not extend as gracefully to completely non-parametric formulations -- where typical observations provide very little information about the tail of the offer distribution. (That point is the main focus in \cite{goldenshluger2017optimal}.) 

Our interest in this case-study of optimal stopping is partially driven by that problem per se, and open questions that pertain, but more broadly  in learning challenges in  sequential decision problems. While generic finite state and action Markov decision processes have served as the common baseline in theoretical reinforcement learning, features of the optimal stopping problem may translate more naturally to other core problems in operations research, like inventory control, queuing admission control, dynamic pricing, and sequential auctions. In particular, in these problems the search for effective policies often reduces to carefully adjusting thresholds (much like the main characteristic of the full information  policy in optimal stopping).  In those contexts, the decision-maker may use data to estimate particular distributions, like the distribution of arrivals in a queuing system, and these distributions are linked to state transitions through known structural equations. In such settings,  more passive forms of exploration may suffice, in contrast to the temporally extended periods of active exploration required to reach poorly understood states in, say,  a tabular reinforcement learning problem. Hence, constructing customized solutions to structured problems remains an interesting area of research in reinforcement learning.

\bibliography{ref}

\begin{thebibliography}{17}
\providecommand{\natexlab}[1]{#1}
\providecommand{\url}[1]{\texttt{#1}}
\expandafter\ifx\csname urlstyle\endcsname\relax
  \providecommand{\doi}[1]{doi: #1}\else
  \providecommand{\doi}{doi: \begingroup \urlstyle{rm}\Url}\fi

\bibitem[Agarwal et~al.(2020)Agarwal, Kakade, and Yang]{agarwal2020model}
Alekh Agarwal, Sham Kakade, and Lin~F Yang.
\newblock Model-based reinforcement learning with a generative model is minimax
  optimal.
\newblock In \emph{Conference on Learning Theory}, pages 67--83. PMLR, 2020.

\bibitem[Azar et~al.(2017)Azar, Osband, and Munos]{azar2017minimax}
Mohammad~Gheshlaghi Azar, Ian Osband, and R{\'e}mi Munos.
\newblock Minimax regret bounds for reinforcement learning.
\newblock In \emph{International Conference on Machine Learning}, pages
  263--272. PMLR, 2017.

\bibitem[Bertsekas(1995)]{bertsekas1995dynamic}
Dimitri~P Bertsekas.
\newblock \emph{Dynamic programming and optimal control}, volume~1.
\newblock Athena scientific Belmont, MA, 1995.

\bibitem[Boshuizen and Hill(1992)]{boshuizen1992moment}
Frans~A Boshuizen and Theodore~P Hill.
\newblock Moment-based minimax stopping functions for sequences of random
  variables.
\newblock \emph{Stochastic processes and their applications}, 43\penalty0
  (2):\penalty0 303--316, 1992.

\bibitem[Ferguson et~al.(1989)]{ferguson1989solved}
Thomas~S Ferguson et~al.
\newblock Who solved the secretary problem?
\newblock \emph{Statistical science}, 4\penalty0 (3):\penalty0 282--289, 1989.

\bibitem[Goldenshluger and Zeevi(2021)]{goldenshluger2017optimal}
A~Goldenshluger and A~Zeevi.
\newblock Optimal stopping of a random sequence with unknown distribution.
\newblock \emph{Mathematics of Operations Research (to appear)}, 2021.

\bibitem[Jaksch et~al.(2010)Jaksch, Ortner, and Auer]{jaksch2010near}
Thomas Jaksch, Ronald Ortner, and Peter Auer.
\newblock Near-optimal regret bounds for reinforcement learning.
\newblock \emph{Journal of Machine Learning Research}, 11\penalty0 (4), 2010.

\bibitem[Jin et~al.(2018)Jin, Allen-Zhu, Bubeck, and Jordan]{jin2018q}
Chi Jin, Zeyuan Allen-Zhu, Sebastien Bubeck, and Michael~I Jordan.
\newblock Is \uppercase{Q}-learning provably efficient?
\newblock In \emph{Proceedings of the 32nd International Conference on Neural
  Information Processing Systems}, pages 4868--4878, 2018.

\bibitem[Moser(1956)]{moser1956problem}
Leo Moser.
\newblock On a problem of \uppercase{C}ayley.
\newblock \emph{Scripta Math}, 22:\penalty0 289--292, 1956.

\bibitem[Nilim and El~Ghaoui(2005)]{nilim2005robust}
Arnab Nilim and Laurent El~Ghaoui.
\newblock Robust control of \uppercase{M}arkov decision processes with
  uncertain transition matrices.
\newblock \emph{Operations Research}, 53\penalty0 (5):\penalty0 780--798, 2005.

\bibitem[Petruccelli(1985)]{petruccelli1985maximin}
Joseph~D Petruccelli.
\newblock Maximin optimal stopping for normally distributed random variables.
\newblock \emph{Sankhy{\=a}: The Indian Journal of Statistics, Series A},
  47:\penalty0 36--46, 1985.

\bibitem[Petruccelli et~al.(1980)]{petruccelli1980best}
Joseph~D Petruccelli et~al.
\newblock On a best choice problem with partial information.
\newblock \emph{Annals of Statistics}, 8\penalty0 (5):\penalty0 1171--1174,
  1980.

\bibitem[Samuels(1981)]{samuels1981minimax}
Stephen~M Samuels.
\newblock Minimax stopping rules when the underlying distribution is uniform.
\newblock \emph{Journal of the American Statistical Association}, 76\penalty0
  (373):\penalty0 188--197, 1981.

\bibitem[Stewart(1978)]{stewart1978optimal}
Theodor~J Stewart.
\newblock Optimal selection from a random sequence with learning of the
  underlying distribution.
\newblock \emph{Journal of the American Statistical Association}, 73\penalty0
  (364):\penalty0 775--780, 1978.

\bibitem[Van~der Vaart(2000)]{van2000asymptotic}
Aad~W Van~der Vaart.
\newblock \emph{Asymptotic statistics}, volume~3.
\newblock Cambridge university press, 2000.

\bibitem[Wang et~al.(2020)Wang, Du, Yang, and Kakade]{wang2020long}
Ruosong Wang, Simon~S Du, Lin Yang, and Sham Kakade.
\newblock Is long horizon \uppercase{RL} more difficult than short horizon
  \uppercase{RL}?
\newblock \emph{Advances in Neural Information Processing Systems}, 33, 2020.

\bibitem[Zhang et~al.(2020)Zhang, Ji, and Du]{zhang2020reinforcement}
Zihan Zhang, Xiangyang Ji, and Simon~S Du.
\newblock Is reinforcement learning more difficult than bandits? a near-optimal
  algorithm escaping the curse of horizon.
\newblock \emph{arXiv preprint arXiv:2009.13503}, 2020.

\end{thebibliography}
\bibliographystyle{plainnat}

\newpage 
\appendix

\section{Characterizing the Exponential-Decay Distribution}
This section provides some technical results on the approximation of tailed probability, hazard rate and other metric characterizing the level of concentration of the exponential-decay distribution \eqref{distribution}.
\label{characterizingexp}


Define the inverse CDF as
\[
P_{\theta}(S) = 1-F_{\theta}(S), 
\]
the conditional expectation larger than $S$ as
\[
\mu_{\theta}(S) = \mathbb{E}_{\theta}[X|X>S],
\]
the hazard rate as
\[
h_{\theta}(S) = \frac{f_{\theta}(S)}{P_{\theta}(S)}
\]
and $$g_\theta(S) := \frac{P_{\theta}(S)}{\int_S^{+\infty} P_{\theta}(x) dx}.$$

\begin{restatable}[Bounds on Tailed Probability]{lemma}{Bounds on Tailed Probability}
\label{Bounds on Tailed Probability}
Consider any $\theta \in \mathbb{R}$  and $S > \theta + 1$.  
\begin{itemize}
    \item If $\alpha \geq 1,$ then 
    $$C_0\cdot \left(\frac{1}{\alpha (S-\theta)^{\alpha-1}} - \frac{\alpha-1}{\alpha^2 (S-\theta)^{2\alpha-1}}\right) e^{-(S-\theta)^\alpha} \leq P_{\theta}(S)\leq \frac{C_0}{\alpha (S-\theta)^{\alpha-1}}e^{-(S-\theta)^\alpha}.$$
    \item If $\alpha \in [1/2, 1)$, then
    $$\frac{C_0}{\alpha (S-\theta)^{\alpha-1}}e^{-(S-\theta)^\alpha} \leq P_{\theta}(S)\leq C_0\cdot \left(\frac{1}{\alpha (S-\theta)^{\alpha-1}} + \frac{1-\alpha}{\alpha^2 (S-\theta)^{2\alpha-1}}\right) e^{-(S-\theta)^\alpha}.$$
\end{itemize}
\end{restatable}

\begin{proof}
For $\alpha \geq 1,$ we have
$$\begin{aligned}
P_{\theta}(S) &= \int_{(S-\theta)^{\alpha}}^{+\infty} \frac{C_0}{\alpha y^{\frac{\alpha-1}{\alpha}}} e^{-y} dy
\\ &\leq \frac{C_0}{\alpha (S-\theta)^{\alpha-1}} \int_{(S-\theta)^{\alpha}}^{+\infty} e^{-y} dy
\\ &= \frac{C_0}{\alpha (S-\theta)^{\alpha-1}}e^{-(S-\theta)^\alpha}.
\end{aligned}$$
On the other hand,
$$\begin{aligned}
P_{\theta}(S) &\geq C_0\cdot \int_{(S-\theta)^{\alpha}}^{+\infty} \left(\frac{1}{\alpha y^{\frac{\alpha-1}{\alpha}}} - \frac{(\alpha-1)(2\alpha-1)}{\alpha^3 y^{\frac{3\alpha -1}{\alpha}}}\right) e^{-y} dy
\\ &= C_0\cdot \left[\frac{\alpha-1}{\alpha^2 y^{\frac{2\alpha-1}{\alpha}}} - \frac{1}{\alpha y^{\frac{\alpha-1}{\alpha}}}\right]^{+\infty}_{(S-\theta)^\alpha} 
\\ &= C_0\cdot \left(\frac{1}{\alpha (S-\theta)^{\alpha-1}} - \frac{\alpha-1}{\alpha^2 (S-\theta)^{2\alpha-1}}\right) e^{-(S-\theta)^\alpha}.
\end{aligned}$$

For $\frac{1}{2} \leq \alpha <1,$ we can obtain the same type of result by changing the directions of two inequalities and swapping the upper and lower bounds.
\end{proof}

Next we consider a metric of a distribution: $$g_\theta(S) = \frac{P_{\theta}(S)}{\int_S^{+\infty} P_{\theta}(x) dx}.$$
Similar to hazard rate, $g_\theta(S)$ also characterizes the decay rate of the tail. This can be formalized by Lemma \ref{lemma3.1}.
\begin{restatable}{lemma}{lemma3.1}
\label{lemma3.1}
For any $S\in \mathbb{R}$,
$$g_{\theta}(S) = \frac{1}{\mu_{\theta}(S)-S}.$$
\end{restatable}

\begin{proof}
Note that for any $S\in \mathbb{R}$,
$$\begin{aligned}
\int_S^{\infty} P_{\theta}(x) dx &= \int_S^{\infty} \left(\int_x^{\infty} f_{\theta}(t) dt\right) dx
\\ &= \int_S^{\infty} \left(\int_S^{t} f_{\theta}(t) dx\right) dt
\\ &= \int_S^{\infty} (t-S)f_{\theta}(t) dt
\\ &= P_{\theta}(S)\cdot \left(\mu_{\theta}(S) -S\right).
\end{aligned}$$
\end{proof}

Now we can state the asymptotic approximation results on $h_\theta(S)$ and $g_\theta(S)$. Note that Lemma \ref{uniform g and h} holds uniformly for all location parameters $\theta \in \mathbb{R}$.
\begin{restatable}{lemma}{uniform g and h}
\label{uniform g and h}
For any $\theta \in \mathbb{R}$, as $S \rightarrow \infty$,
$$h_{\theta}(S) \sim \alpha (S-\theta)^{\alpha-1},$$
$$g_{\theta}(S) \sim \alpha (S-\theta)^{\alpha-1}.$$
Moreover, as $T \rightarrow +\infty$,
$$\sup_{\theta \in \mathbb{R}}{\left| \frac{h_{\theta}(T + \theta)}{\alpha T^{\alpha-1}}-1\right|} \longrightarrow 0,$$ $$\sup_{\theta \in \mathbb{R}}{\left| \frac{g_{\theta}(T + \theta)}{\alpha T^{\alpha-1}}-1\right|} \longrightarrow 0.$$
\end{restatable}

\begin{proof}We first consider the case $\alpha \geq 1$. By Lemma \ref{Bounds on Tailed Probability},
\begin{equation}
\label{bound8.3.1}
    \alpha (S-\theta)^{\alpha-1} \leq \frac{f_{\theta}(S)}{P_{\theta}(S)} \leq \frac{1}{\left(\frac{1}{\alpha (S-\theta)^{\alpha-1}} - \frac{\alpha-1}{\alpha^2 (S-\theta)^{2\alpha-1}}\right)},
\end{equation}
thus as $S-\theta \rightarrow +\infty$, $$h_{\theta}(S) \sim \alpha (S-\theta)^{\alpha-1}.$$

Now we consider $g_{\theta}(S) = \frac{P_{\theta}(S)}{\int_S^{\infty} P_{\theta}(x) dx}$. By Lemma \ref{Bounds on Tailed Probability}:
$$g_{\theta}(S) \geq \frac{C_0\cdot \left(\frac{1}{\alpha (S-\theta)^{\alpha-1}} - \frac{\alpha-1}{\alpha^2 (S-\theta)^{2\alpha-1}}\right) e^{-(S-\theta)^\alpha}}{\int_S^{\infty} \frac{C_0}{\alpha (x-\theta)^{\alpha-1}}e^{-(x-\theta)^\alpha} dx},$$
and
$$g_{\theta}(S) \leq \frac{\frac{C_0}{\alpha (S-\theta)^{\alpha-1}}e^{-(S-\theta)^\alpha}}{\int_S^{\infty} C_0\cdot \left(\frac{1}{\alpha (x-\theta)^{\alpha-1}} - \frac{\alpha-1}{\alpha^2 (x-\theta)^{2\alpha-1}}\right) e^{-(x-\theta)^\alpha} dx}.$$
By the same type of argument presented in Lemma \ref{Bounds on Tailed Probability}, we know that as $S-\theta \rightarrow +\infty$,
\begin{equation} \label{twolongbound1}
    \int_S^{\infty} \frac{C_0}{\alpha (x-\theta)^{\alpha-1}}e^{-(x-\theta)^\alpha} dx \sim \frac{C_0}{\alpha^2 (S-\theta)^{2\alpha-2}}e^{-(S-\theta)^\alpha},
\end{equation}
\begin{equation} \label{twolongbound2}
    \int_S^{\infty} C_0\cdot \left(\frac{1}{\alpha (x-\theta)^{\alpha-1}} - \frac{\alpha-1}{\alpha^2 (x-\theta)^{2\alpha-1}}\right) e^{-(x-\theta)^\alpha} dx \sim \frac{C_0}{\alpha^2 (S-\theta)^{2\alpha-2}}e^{-(S-\theta)^\alpha},
\end{equation}
hence
$$g_{\theta}(S) \sim \alpha (S-\theta)^{\alpha-1}.$$

For the case $\frac{1}{2} \leq \alpha <1,$ we can obtain the same result by similar argument.

The uniform convergence is a direct consequence of the fact that both $h_{\theta}(S) = h_0(S-\theta)$ and $g_{\theta}(S) = g_0(S-\theta)$ are functions of $(S-\theta)$. Let $T = S-\theta$, we finish the proof.
\end{proof}

Based on Lemma \ref{lemma3.1} and Lemma \ref{uniform g and h}, we immediately get the following corollary.

\begin{restatable}{coro}{littlehittingproblemma}
\label{littlehittingproblemma}
For any $\theta \geq 0$,
\begin{equation}
\label{expo(S)}
    \lim_{S\rightarrow +\infty} \frac{\mu_{\theta}(S) - S}{(S-\theta)^{1-\alpha}} = \frac{1}{\alpha}.
\end{equation}
\end{restatable}

\begin{restatable}{rmk}{frac12}
\label{frac12}
For any $\alpha >0$, the exponential-decay distribution $F_{\theta}$ satisfies $\mu(S) - S = o(S)$. The restriction $\alpha \geq \frac{1}{2}$ is just to make our proof concise. On the other hand, if $F_0$ is extremely heavy-tailed (polynomial-tailed), then $\mu(S) - S = \Theta(S)$.
\end{restatable}

\begin{restatable}{rmk}{example1}
\label{example1}
Consider a class of distributions sharing heavier tails than exponential-decay distribution described in \eqref{distribution}. For any $\beta >2$, the heavy-tailed distribution $G_\beta$ has density function
$$
g_\beta(x)=\left\{\begin{array}{ll}\frac{\beta-1}{x^\beta}, & x \geq 1 \\ 0, & x < 1 \end{array}\right.
$$
it holds that $$\mu (S) - S = \frac{1}{\beta-2}\cdot S$$ for any $S \geq 1$. Thus $\mu(S) - S = \Theta(S)$ for any distribution $G_\beta,\,\, \beta >2$.
\end{restatable}

In the end of this section, we provide bounds on the first and second derivatives of $\log{P_{\theta}(S)}$ with respect to $S$.

\begin{restatable}{lemma}{qiudaobudengshi}
\label{qiudaobudengshi}
For any $\theta \in \mathbb{R}$ and $S > \theta$, 
$$\frac{\partial \log{P_{\theta}(x)}}{\partial x}(S) = -h_{\theta}(S).$$
If $S > \theta +1$ and $\alpha \geq 1$, then
    \begin{equation}
    \left|\frac{\partial^2 \log{P_{\theta}(x)}}{\partial x^2}(S)\right| \leq \frac{\alpha^3 (\alpha-1) (S-\theta)^{3\alpha-2}}{\left(\alpha(S-\theta)^{\alpha} -(\alpha-1) \right)^2}.
\end{equation}
If $S> \theta +1$ and $\alpha \in [1/2, 1)$, then 
    \begin{equation}
    \left|\frac{\partial^2 \log{P_{\theta}(x)}}{\partial x^2}(S)\right| \leq \alpha(1-\alpha)(S-\theta)^{\alpha-2}.
\end{equation}
\end{restatable}

\begin{proof}
The first derivative satisfies
$$\frac{\partial \log{P_{\theta}(x)}}{\partial x}(S) = \frac{-f_{\theta}(S)}{P_{\theta}(S)} = -h_{\theta}(S).$$

For the second derivative,
\begin{equation}
\label{lemma4.2dafenshi}
    \frac{\partial^2 \log{P_{\theta}(x)}}{\partial x^2}(S) = - \frac{P_{\theta}(S)\cdot \frac{\partial f_{\theta}(x)}{\partial x}(S) + f_{\theta}(S)^2}{P_{\theta}(S)^2},
\end{equation}

Notice that $\frac{\partial f_{\theta}(x)}{\partial x} = -\alpha(x-\theta)^{\alpha-1} f_{\theta}(x)$, when $S > \theta +1$, it holds that $\frac{\partial f_{\theta}(x)}{\partial x}(S) <0$. Combining with Lemma \ref{Bounds on Tailed Probability}, we obtain the following bound on numerator for any $\alpha \geq \frac{1}{2}$
$$\begin{aligned}
\left|P_{\theta}(S)\cdot \frac{\partial f_{\theta}(x)}{\partial x}(S) + f_{\theta}(S)^2\right| &= |\alpha(S-\theta)^{\alpha-1}P_{\theta}(S) - f_{\theta}(S)|\cdot f_{\theta}(S)
\\ &\leq \alpha(S-\theta)^{\alpha-1}\cdot \frac{|\alpha-1|}{\alpha^2 (S-\theta)^{2\alpha-1}}\cdot f_{\theta}(S)^2
\\ &= \frac{|\alpha-1|}{\alpha (S-\theta)^{\alpha}}\cdot f_{\theta}(S)^2.
\end{aligned}$$
The inequality in the second line is due to the fact that for any $\alpha \geq \frac{1}{2}$, it holds that 
$$\left|\frac{P_{\theta}(S)}{f_{\theta}(S)} - \frac{1}{\alpha (S-\theta)^{\alpha-1}}\right| \leq \frac{|\alpha-1|}{\alpha^2 (S-\theta)^{2\alpha-1}},$$ which is a direct result of equation \eqref{bound8.3.1} and its variant version for $\frac{1}{2} \leq \alpha < 1$.

For the denominator, if $\alpha \geq 1,$
$$P_{\theta}(S)^2 \geq \left(\frac{1}{\alpha (S-\theta)^{\alpha-1}} - \frac{\alpha-1}{\alpha^2 (S-\theta)^{2\alpha-1}} \right)^2\cdot f_{\theta}(S)^2;$$
if $\frac{1}{2} \leq \alpha <1,$
$$P_{\theta}(S)^2 \geq \left(\frac{1}{\alpha (S-\theta)^{\alpha-1}} \right)^2\cdot f_{\theta}(S)^2.$$
Plugging them into the equation \eqref{lemma4.2dafenshi}, we have for $\alpha \geq 1$,
$$\begin{aligned}
\left|\frac{\partial^2 \log{P_{\theta}(x)}}{\partial x^2}(S)\right| &\leq \frac{\frac{\alpha-1}{\alpha (S-\theta)^{\alpha}}}{\left(\frac{1}{\alpha (S-\theta)^{\alpha-1}} - \frac{\alpha-1}{\alpha^2 (S-\theta)^{2\alpha-1}} \right)^2}
\\ &= \frac{\alpha^3 (\alpha-1) (S-\theta)^{3\alpha-2}}{\left(\alpha(S-\theta)^{\alpha} -(\alpha-1) \right)^2};
\end{aligned}$$
for $\frac{1}{2}\leq \alpha < 1$,
$$\begin{aligned}
\left|\frac{\partial^2 \log{P_{\theta}(x)}}{\partial x^2}(S)\right| &\leq \frac{\frac{1-\alpha}{\alpha (S-\theta)^{\alpha}}}{\left(\frac{1}{\alpha (S-\theta)^{\alpha-1}} \right)^2}
\\ &= \alpha(1-\alpha)(S-\theta)^{\alpha-2};
\end{aligned}$$
\end{proof}

\section{Uniform Approximation of the Optimal Threshold}
\label{uniformapprox}
In this section we present the uniform approximation result on Bellman threshold value for exponential-decay distribution, which is essential in the proof of Proposition \ref{prop:stopping_prob} in Appendix \ref{proofprop2}. Basically, the solution to Bellman Equation \eqref{threhsoldbellman} can be approximated by a logarithmic term dependent only on the horizon length ${1}/{(1-\gamma)}$ plus the location parameter $\theta$. Moreover, this approximate decomposition holds uniformly for all $\theta$ in an expanding parameter space. To clarify this point, we consider the case $\theta$ is bounded by an slowly exploding function $f(\gamma)$, which is a non-negative function of $\gamma$ satisfying 
\begin{equation}
\label{thetabound}
    f(\gamma) = o\left(\left(\log{\frac{1}{1-\gamma}}\right)^{\frac{1}{\alpha}}\right).
\end{equation}

\begin{restatable}{lemma}{Lemma5.1}
\label{Lemma5.1}
If $f(\gamma)$ satisfies \eqref{thetabound}, then for any $\theta_0 \in \mathbb{R}$, as $\gamma \rightarrow 1$, 
\begin{equation}
\label{newsandwich}
    \sup_{\theta \in \left[\theta_0-f(\gamma),\,\, \theta_0+f(\gamma)\right]} \left|\frac{S^*({\theta},\gamma) - \theta}{\left(\log{\frac{1}{1-\gamma}}\right)^{\frac{1}{\alpha}}} - 1\right| \longrightarrow 0.
\end{equation}
\end{restatable}

\begin{proof}
Without loss of generality we can assume $\theta_0 = 0$ since we only consider the asymptotic results for fixed $\theta_0$. First we consider the case $\alpha \geq 1$. We consider the integration of tailed probability $\int_{S}^{+\infty} P_{\theta}(x) dx$ for any $S, \theta$. By Lemma \ref{Bounds on Tailed Probability}, we know that:
$$\begin{aligned}
\int_S^{\infty} C_0\cdot \left(\frac{1}{\alpha (x-\theta)^{\alpha-1}} - \frac{\alpha-1}{\alpha^2 (x-\theta)^{2\alpha-1}}\right) e^{-(x-\theta)^\alpha} dx &\leq \int_{S}^{+\infty} P_{\theta}(x) dx 
\\&\leq \int_S^{\infty} \frac{C_0}{\alpha (x-\theta)^{\alpha-1}}e^{-(x-\theta)^\alpha} dx.
\end{aligned}$$
According to the equation \eqref{twolongbound1} and equation \eqref{twolongbound2}, if $S-\theta \geq T(\alpha, C_0)$ where $T(\alpha, C_0)$ is only dependent on $\alpha$ and $C_0$, the following bounds hold:
\begin{equation}
\label{sandwichbellman}
    \frac{1}{2}\cdot \frac{C_0}{\alpha^2 (S-\theta)^{2\alpha-2}}e^{-(S-\theta)^\alpha} \leq \int_{S}^{+\infty} P_{\theta}(x) dx \leq \frac{2C_0}{\alpha^2 (S-\theta)^{2\alpha-2}}e^{-(S-\theta)^\alpha}.
\end{equation}
For any $c_0 >0$, consider the unique solution $S = S_1^{\prime}(\theta, \gamma)$ to the equation
\begin{equation}
\label{eq519}
    \int_{S}^{+\infty}P_{\theta}(x) d x = (1-\frac{2c_0}{3})\cdot \left(\log{\frac{1}{1-\gamma}}\right)^{\frac{1}{\alpha}} \frac{1-\gamma}{\gamma}.
\end{equation}
Since the left-hand side of equation \eqref{eq519} is a monotonically decreasing function of $(S-\theta)$, and the right-hand side of equation \eqref{eq519} is a monotonically decreasing function of $\gamma$, there exists $\gamma_0\prime$ such that for every $\gamma \geq \gamma_0\prime$, any $\theta \in \mathbb{R}$, the solution $S_1^{\prime}(\theta, \gamma)$ satisfies $S_1^{\prime}(\theta, \gamma) - \theta \geq T(\alpha, C_0)$, hence \eqref{sandwichbellman} holds. From now on we without loss of generality assume $\gamma \geq \gamma_0\prime$. Combining the left-hand side bound in \eqref{sandwichbellman} and equation \eqref{eq519}, we have
$$\frac{1}{2}\cdot \frac{C_0}{\alpha^2 (S_1^{\prime}(\theta, \gamma)-\theta)^{2\alpha-2}}e^{-(S_1^{\prime}(\theta, \gamma)-\theta)^\alpha} \leq \left(\log{\frac{1}{1-\gamma}}\right)^{\frac{1}{\alpha}} \frac{1-\gamma}{\gamma}.$$
Notice that as $\gamma \rightarrow 1$, $S_1^{\prime}(\theta, \gamma) \rightarrow +\infty$, and the left-hand side is dominated by the exponential term $e^{-(S_1^{\prime}(\theta, \gamma)-\theta)^\alpha}$ and right-hand side is dominated by the linear term $1-\gamma$. Therefore, there exists $\gamma_1$ such that for every $\gamma \geq \gamma_1$,
$$
    S_1^{\prime}(\theta,\gamma) \geq \theta + (1-\frac{c_0}{3})\cdot \left(\log{\frac{1}{1-\gamma}}\right)^{\frac{1}{\alpha}},\,\,\, \text{for any $\theta \in \mathbb{R}$.}
$$
According to the definition of $f(\gamma)$, there exists $\gamma_2$ such that for every $\gamma \geq \gamma_2$, it holds that 
\begin{equation}
\label{eq517}
    f(\gamma) \leq \frac{c_0}{3}\cdot \left(\log{\frac{1}{1-\gamma}}\right)^{\frac{1}{\alpha}},
\end{equation}
thus for every $\gamma \geq \gamma_0 = \max{\{\gamma_1,\gamma_2\}}$ and any $\theta \in \left[-f(\gamma), f(\gamma)\right]$, it holds that
\begin{equation}
\label{eq521}
    S_1^{\prime}(\theta,\gamma) \geq (1-\frac{2c_0}{3})\cdot \left(\log{\frac{1}{1-\gamma}}\right)^{\frac{1}{\alpha}}.
\end{equation}
Also consider $S_1(\theta,\gamma) := \theta + (1-c_0)\cdot \left(\log{\frac{1}{1-\gamma}}\right)^{\frac{1}{\alpha}}$. By equation \eqref{eq517}, for any $\theta \in \left[-f(\gamma), f(\gamma)\right]$, it holds that
\begin{equation}
\label{eq522}
    S_1(\theta,\gamma) \leq (1-\frac{2c_0}{3})\cdot \left(\log{\frac{1}{1-\gamma}}\right)^{\frac{1}{\alpha}},
\end{equation}
thus $S_1(\theta,\gamma) \leq S_1^{\prime}(\theta,\gamma)$.
\bigskip

Define function $B(S)$ as
$B(S) := \frac{1}{S} \int_{S}^{+\infty}P_{\theta}(x) d x$. Notice that the integration of tailed probability $\int_{S}^{+\infty}P_{\theta}(x) d x$ and $\frac{1}{S}$ are both monotonically decreasing, it holds that for every $\gamma \geq \gamma_0$ and any $\theta \in \left[-f(\gamma), f(\gamma)\right]$,
$$\begin{aligned}
    B(S_1(\theta,\gamma)) &= \frac{1}{S_1(\theta,\gamma)} \int_{S_1(\theta,\gamma)}^{+\infty}P_{\theta}(x) d x
    \\&\geq \frac{1}{(1-\frac{2c_0}{3})\cdot \left(\log{\frac{1}{1-\gamma}}\right)^{\frac{1}{\alpha}}} \int_{S_1^{\prime}(\theta,\gamma)}^{+\infty}P_{\theta}(x) d x
    \\&= \frac{1-\gamma}{\gamma},
\end{aligned}$$
the last equality holds due to equation \eqref{eq519}. 

Notice that the optimal threshold $S^*(\theta,\gamma)$ satisfies the equation $B(S^*(\theta,\gamma)) = \frac{1-\gamma}{\gamma}$, and $B(S)$ is monotonically decreasing for $S>0$, thus for every $\gamma \geq \gamma_0$ and any $\theta \in \left[-f(\gamma), f(\gamma)\right]$,
$$S^*(\theta,\gamma) \geq S_1(\theta,\gamma) = \theta + (1-c_0)\cdot \left(\log{\frac{1}{1-\gamma}}\right)^{\frac{1}{\alpha}}.$$
By an almost same argument, we obtain the bound on other side: for every $\gamma \geq \gamma_0^{\prime}$ and any $\theta \in \left[-f(\gamma), f(\gamma)\right]$,
$$S^*(\theta,\gamma) \leq S_2(\theta,\gamma) = \theta + (1+c_0)\cdot \left(\log{\frac{1}{1-\gamma}}\right)^{\frac{1}{\alpha}}.$$
By the arbitrariness of selection of $c_0$, we obtain equation \eqref{newsandwich} as $\gamma \rightarrow 1$ for $\alpha \geq 1$. The case $\frac{1}{2} \leq \alpha < 1$ can be proved by an almost same argument.
\end{proof}

The key purpose of Lemma \ref{Lemma5.1} is that we can give a uniform bound on the derivative of the threshold with respect to the location parameter $\theta$ (Lemma \ref{lemmalocationshift}), which can help to relate the difference of two Bellman thresholds with that of the location parameters (Corollary \ref{coro5.6}).

Note that unlike typical uniform convergence results, our result holds uniformly for an varying universe, i.e., as $\gamma$ gets closer to 1, $S^*(\theta, \gamma)$ can be approximated by $\theta +\left(\log{\frac{1}{1-\gamma}}\right)^{\frac{1}{\alpha}}$ over an expanding region in which parameter $\theta$ lies. Based on that, we can further obtain an approximation result on $g_{\theta}(S^*(\theta, \gamma))$ and $S^*(\theta, \gamma)g_{\theta}(S^*(\theta, \gamma))$.

\begin{restatable}{coro}{proposition5.2}
\label{proposition5.2}
If $f(\gamma)$ satisfies \eqref{thetabound}, then for any $\theta_0 \in \mathbb{R}$, as $\gamma \rightarrow 1$, 
\begin{equation}
\label{5.4.1}
    \sup_{\theta \in [\theta_0-f(\gamma),\,\, \theta_0+f(\gamma)]} \left|\frac{g_{\theta}(S^*(\theta, \gamma))}{\alpha \left(\log{\frac{1}{1-\gamma}}\right)^{1-\frac{1}{\alpha}}} - 1\right| \longrightarrow 0,
\end{equation}
and
\begin{equation}
\label{5.4.2}
    \sup_{\theta \in [\theta_0-f(\gamma),\,\, \theta_0+f(\gamma)]} \left|\frac{S^*(\theta, \gamma)g_{\theta}(S^*(\theta, \gamma))}{\alpha \log{\frac{1}{1-\gamma}}} - 1\right| \longrightarrow 0.
\end{equation}
\end{restatable}

\begin{proof}
Without loss of generality we assume $\theta_0 = 0$. For any $\epsilon_1 >0$, by Lemma \ref{uniform g and h}, there exists $T_0 >0$ where $T_0$ not dependent on $\gamma$ such that once $S^*(\theta,\gamma) -\theta \geq T_0$ is satisfied, it holds that
\begin{equation}
\label{5.18}
    \left(1-\epsilon_1\right)^{\frac{1}{2}} \leq \frac{h_{\theta}(S^*(\theta,\gamma))}{\alpha (S^*(\theta,\gamma)-\theta)^{\alpha-1}} \leq \left(1+\epsilon_1\right)^{\frac{1}{2}}.
\end{equation}
By Lemma \ref{Lemma5.1}, there exists $\gamma_1$ such that for every $\gamma \geq \gamma_1$, any $\theta \in [-f(\gamma),\,\, f(\gamma)]$, we have $S^*(\theta,\gamma) -\theta \geq T_0$ hold. Thus \eqref{5.18} holds for any $\gamma \geq \gamma_1$.

On the other hand, by Lemma \ref{Lemma5.1}, there exists $\gamma_2 >0$ such that for every $\gamma \geq \gamma_2$,
$$\sup_{\theta \in [-f(\gamma),\,\, f(\gamma)]} \left(\frac{S^*({\theta},\gamma) - \theta}{\left(\log{\frac{1}{1-\gamma}}\right)^{\frac{1}{\alpha}}}\right)^{\alpha-1} \leq (1+\epsilon_1)^{\frac{1}{2}},$$
and 
$$\inf_{\theta \in [-f(\gamma),\,\, f(\gamma)]} \left(\frac{S^*({\theta},\gamma) - \theta}{\left(\log{\frac{1}{1-\gamma}}\right)^{\frac{1}{\alpha}}}\right)^{\alpha-1} \geq (1-\epsilon_1)^{\frac{1}{2}}.$$
Let $\gamma_0 = \max\{\gamma_1,\gamma_2\}$, then for every $\gamma \geq \gamma_0$, by multiplying \eqref{5.18} and the bounds above, we have:
\begin{equation}
\label{5.19}
    \sup_{\theta \in [-f(\gamma),\,\, f(\gamma)]} \frac{g_{\theta}(S^*(\theta, \gamma))}{\alpha \left(\log{\frac{1}{1-\gamma}}\right)^{1-\frac{1}{\alpha}}} \leq 1+\epsilon_1,
\end{equation}
\begin{equation}
\label{5.20}
    \inf_{\theta \in [-f(\gamma),\,\, f(\gamma)]} \frac{g_{\theta}(S^*(\theta, \gamma))}{\alpha \left(\log{\frac{1}{1-\gamma}}\right)^{1-\frac{1}{\alpha}}} \geq 1-\epsilon_1,
\end{equation}
Therefore we obtain \eqref{5.4.1}.

For any $\epsilon_2>0$, by Lemma \ref{Lemma5.1}, there exists $\gamma_1^{\prime}$ such that for every $\gamma \geq \gamma_1^{\prime}$,
$$ \sup_{\theta \in [-f(\gamma),\,\, f(\gamma)]}\frac{S^*({\theta},\gamma) - \theta}{\left(\log{\frac{1}{1-\gamma}}\right)^{\frac{1}{\alpha}}} \leq 1+\frac{\epsilon_2}{2},$$
$$ \inf_{\theta \in [-f(\gamma),\,\, f(\gamma)]}\frac{S^*({\theta},\gamma) - \theta}{\left(\log{\frac{1}{1-\gamma}}\right)^{\frac{1}{\alpha}}} \geq 1-\frac{\epsilon_2}{2}.$$
Since $f(\gamma) = o\left(\left(\log{\frac{1}{1-\gamma}}\right)^{\frac{1}{\alpha}}\right)$, there exists $\gamma_2^{\prime}$ such that for every $\gamma \geq \gamma_2^{\prime}$, 
$$f(\gamma) \leq \frac{\epsilon_2}{2}\cdot \left(\log{\frac{1}{1-\gamma}}\right)^{\frac{1}{\alpha}}.$$
Thus for every $\gamma \geq \gamma_0^{\prime} = \max\{\gamma_1^{\prime}, \gamma_2^{\prime}\}$,
$$ \sup_{\theta \in [-f(\gamma),\,\, f(\gamma)]}\frac{S^*({\theta},\gamma) }{\left(\log{\frac{1}{1-\gamma}}\right)^{\frac{1}{\alpha}}} \leq 1+\epsilon_2,$$
$$ \inf_{\theta \in [-f(\gamma),\,\, f(\gamma)]}\frac{S^*({\theta},\gamma)}{\left(\log{\frac{1}{1-\gamma}}\right)^{\frac{1}{\alpha}}} \geq 1-\epsilon_2.$$
Combining with \eqref{5.19} and \eqref{5.20}, we obtain \eqref{5.4.2}.
\end{proof}

\medskip

Next we reformulate Lemma \ref{Lemma5.1} through the lens of derivative.

\begin{restatable}{lemma}{lemmalocationshift}
\label{lemmalocationshift}
If $f(\gamma)$ satisfies \eqref{thetabound}, then for any $\theta_0 \in \mathbb{R}$, as $\gamma \rightarrow 1$, 
\begin{equation}
\label{lemma5.5equation}
    \sup_{\theta \in [\theta_0-f(\gamma),\,\, \theta_0+f(\gamma)]} \left|\frac{\partial S^*(\theta,\gamma)}{\partial \theta} - 1\right| \longrightarrow 0.
\end{equation}
\end{restatable}

\begin{proof}
Without loss of generality we assume $\theta_0 = 0$. Note that the Bellman Equation is 
$$
\frac{1}{\gamma}S^*(\theta,\gamma)=\frac{1}{1-\gamma} \int_{S^*(\theta,\gamma)}^{+\infty}P_{\theta}(x) d x,$$
By Lemma \ref{Lemma5.1}, there exists $\gamma_0\prime$ such that for every $\gamma \geq \gamma_0\prime$, for any $\theta \in \mathbb{R}$, it holds that $S^*(\theta,\gamma) \geq \theta +1$. We know that
$$\frac{\partial P_{\theta}(x)}{\partial \theta} = -\frac{\partial P_{\theta}(x)}{\partial x} = f_{\theta}(x).$$
By Leibniz integral rule, we have:
$$\frac{1-\gamma}{\gamma}\cdot \frac{\partial S^*(\theta,\gamma)}{\partial \theta} = - \frac{\partial S^*(\theta,\gamma)}{\partial \theta}\cdot P_{\theta}(S^*(\theta,\gamma)) + \int_{S^*(\theta,\gamma)}^{+\infty}f_{\theta}(x) d x,$$
Thus we have the optimal threshold $S^*(\theta,\gamma)$ satisfy:
\begin{equation}
\label{bellmandiffrential}
    \frac{\partial S^*(\theta,\gamma)}{\partial \theta} = \frac{P_{\theta}(S^*(\theta,\gamma))}{P_{\theta}(S^*(\theta,\gamma)) +\frac{1-\gamma}{\gamma}}.
\end{equation}

By the definition of $g_{\theta}(S^*(\theta,\gamma))$, the Bellman Equation is also equivalent to
$$P_{\theta}(S^*(\theta,\gamma)) = \frac{S^*(\theta,\gamma)g_{\theta}(S^*(\theta,\gamma))}{\gamma}\cdot (1-\gamma).$$
By plugging it into equation \eqref{bellmandiffrential}, we have $$\frac{\partial S^*(\theta,\gamma)}{\partial \theta} = 1-
\frac{1}{S^*(\theta,\gamma)g_{\theta}(S^*(\theta,\gamma)) +1} \leq 1.$$
For any $\epsilon_0 >0$, there exists $\gamma_1$ such that $\alpha \log{\frac{1}{1-\gamma}} \geq 2/\epsilon_0$. Meanwhile by Corollary \ref{proposition5.2}, there exists $\gamma_2$ such that for every $\gamma \geq \gamma_2$,
$$\inf_{\theta \in [-f(\gamma),\,\, f(\gamma)]} \frac{S^*(\theta, \gamma)g_{\theta}(S^*(\theta, \gamma))}{\alpha \log{\frac{1}{1-\gamma}}} \geq \frac{1}{2}.$$
Therefore for every $\gamma \geq \gamma_0 = \max\{\gamma_1, \gamma_2\}$,
$$\sup_{\theta \in [-f(\gamma),\,\, f(\gamma)]} \left|\frac{\partial S^*(\theta,\gamma)}{\partial \theta} - 1\right| \leq \epsilon_0.$$
\end{proof}

One immediate consequence of Lemma \ref{lemmalocationshift} is that the perturbation of location parameter leads to an almost same extent of perturbation on the optimal threshold as $\gamma$ close to 1 enough, Moreover, it is in the uniform sense. Thus it provides a good approximation of the difference of two optimal thresholds.

\begin{restatable}{coro}{coro5.6}
\label{coro5.6}
For fixed $\alpha \geq \frac{1}{2}$ and $\theta_0 \in \mathbb{R}$, suppose $f(\gamma)$ satisfies \eqref{thetabound}. Then for any constant $c_0 >0$, there exist $\gamma_0$ such that for every $\gamma \geq \gamma_0$, the following bound holds uniformly for any $\theta \in [\theta_0 - f(\gamma),\,\, \theta_0 + f(\gamma)]$:
    \begin{equation}
    \label{5.6approx}
        (1-c_0)\cdot (\theta-\theta_0) < S^*(\theta,\gamma) - S^*(\theta_0,\gamma) <(1+c_0)\cdot (\theta-\theta_0).
    \end{equation}
\end{restatable}

\begin{proof}
Without loss of generality we assume $\theta_0 = 0$. For any $\theta \in [- f(\gamma),\,\, f(\gamma)]$ and $\gamma$, by Mean Value Theorem, we have
$$S^*(\theta,\gamma) - S^*(\theta_0,\gamma) = \frac{\partial S^*(\theta,\gamma)}{\partial \theta}(\Tilde{\theta}) \cdot (\theta-\theta_0),$$
where $\Tilde{\theta} \in [- f(\gamma),\,\, f(\gamma)].$ Meanwhile, by Lemma \ref{lemmalocationshift}, for any constant $c_0>0$, there exists $\gamma_0$ such that for every $\gamma \geq \gamma_0$, it holds that
$$\sup_{\Tilde{\theta} \in [- f(\gamma),\,\, f(\gamma)]} \left|\frac{\partial S^*(\theta,\gamma)}{\partial \theta}(\Tilde{\theta}) - 1\right| \leq c_0.$$
Thus we finish the proof.
\end{proof}

\bigskip

\section{Proofs of Theorem \ref{thmpositive} and Theorem \ref{thmnegative}}
\label{proofthmperturbation}
Besides Proposition \ref{prop:stopping_prob}, the proofs of Theorem \ref{thmpositive} and Theorem \ref{thmnegative} rely on the following Proposition \ref{phasetransitionlemma}. This result actually holds on under a broader set of distributions than those in \eqref{distribution}. In particular, we will later show that \eqref{mus-s} and \eqref{almostsamethreshold} hold for our class under the exponential decay distributions in \eqref{distribution}.  The more abstract presentation here can be thought of as identifying properties that seem critical. To emphasize that tail probability and conditional expectation are evaluated under the ground truth $F_0$, we denote $\mathbb{E}_{F_0}[X|X>S]$ as $\mu_{0}(S)$ and $\mathbb{P}_{F_0}[X>S]$ as $P_{0}(S)$. For any distribution $F$, we denote the solution to Bellman Equation \eqref{threhsoldbellman} as $S^*(F,\gamma)$ in Proposition \ref{phasetransitionlemma}.

\begin{restatable}{prop}{phasetransitionlemma}
\label{phasetransitionlemma}
Suppose $F_0$ and $\{F_{\gamma}\}_{0<\gamma<1}$ satisfies
\begin{equation}
\label{mus-s}
    \mu_0(S) \sim S
\end{equation}
as $S\rightarrow +\infty$ and Bellman thresholds $S^*(F_\gamma,\gamma)$ satisfy
\begin{equation}
\label{almostsamethreshold}
    S^*(F_\gamma,\gamma) \sim S^*(F_0,\gamma)
\end{equation}
as $\gamma \rightarrow 1$.
\begin{itemize}
    \item If $P_{0}(S^*(F_\gamma,\gamma)) = \omega (1-\gamma)$, then $$\lim_{\gamma \rightarrow 1} \Rc(F_0, \gamma, \tau^*(F_\gamma)) = 0.$$
    \item If $P_{0}(S^*(F_\gamma,\gamma)) = o(1-\gamma)$, then $$\lim_{\gamma \rightarrow 1} \Rc(F_0, \gamma, \tau^*(F_\gamma)) = 1.$$
\end{itemize}
\end{restatable}

The proofs of the two propositions will be given in Appendix \ref{proofprop1} and Appendix \ref{proofprop2} respectively. Now we recall Theorem \ref{thmpositive}:
\thmpositive*
\begin{proof}
By Corollary \ref{littlehittingproblemma}, \eqref{mus-s} holds for $F_{\theta}$. Meanwhile by Lemma \ref{Lemma5.1}, \eqref{almostsamethreshold} holds for $F_{\theta}$. According to Proposition \ref{prop:stopping_prob} and Proposition \ref{phasetransitionlemma}, we obtain the result.
\end{proof}

Recall Theorem \ref{thmnegative}:
\thmnegative*
\begin{proof}
Notice that for negative perturbation $\epsilon(\gamma) <0$, $S^*(\theta+\epsilon(\gamma),\gamma) < S^*(\theta,\gamma)$. Meanwhile by Corollary \ref{prop4.2}, it holds that $P_{\theta} (S^*(\theta,\gamma)) = \omega(1-\gamma)$, therefore
$$P_{\theta} (S^*(\theta+\epsilon(\gamma),\gamma)) = \omega(1-\gamma).$$
By Proposition \ref{phasetransitionlemma}, we obtain $\lim_{\gamma \rightarrow 1} \Rc(F_{\theta}, \gamma, \tau^*(\theta+\epsilon(\gamma))) = 0$.
\end{proof}

\section{Proof of Main Theorems}
\label{proofestimationthm}
In this section we prove Theorem \ref{estimationthm} and Theorem \ref{estimationthm_onesample}. 

We first introduce a lemma to show that the maximum likelihood estimator is asymptotically normal. We use the following result on the asymptotic normality of M-estimators. In the notation of \cite{van2000asymptotic}, $P$ denotes the population distribution from which data is drawn and, viewing this as a linear functional, he use writes $Pg=\E_{X\sim P}[g(X)]$. The important feature of this result, for our purposes, is that he requires a Taylor expansion of $Pm_{\theta}$ rather than $m_{\theta}$ itself. This will allow us to smooth out the discontinuity of the log-likelihood $\log f_{\theta}(\cdot)$ through integration. We use the notation $\stackrel{\mathrm{P}}{\rightarrow}$ or  $\stackrel{\mathrm{d}}{\rightarrow}$ to denote convergence in probability and convergence in distribution, respectively. 
\begin{thm*}[Theorem 5.23 of \cite{van2000asymptotic}]
For each $\theta$ in an open subset of Euclidean space let $x \mapsto m_{\theta}(x)$ be a measurable function such that $\theta \mapsto m_{\theta}(x)$ is differentiable at $\theta_0$ for $P-$almost every $x$ with derivative $\dot{m}_{\theta_{0}}(x)$ and such that, for every $\theta_1$ and $\theta_2$ in a neighborhood of $\theta_0$ and a measurable function $\dot{m}$ with $P \dot{m}^{2}<\infty$
\begin{equation}\label{eq:vdv1}
\left|m_{\theta_{1}}(x)-m_{\theta_{2}}(x)\right| \leq \dot{m}(x)\left\|\theta_{1}-\theta_{2}\right\|.    
\end{equation}
Furthermore, assume that the map $\theta \mapsto P m_{\theta}$ admits a second-order Taylor expansion at a point of maximum $\theta_0$, meaning 
\begin{equation}\label{eq:vdv2}
Pm_{\theta} = Pm_{\theta_0} + \frac{1}{2} (\theta - \theta_0)^\top V_{\theta_0} (\theta - \theta_0) + o(\| \theta - \theta_0\|^2),
\end{equation}
with nonsingular symmetric second derivative matrix $V_{\theta_0}$. If $\hat{\theta}_n \in \arg\max_{\theta} n^{-1}\sum_{i=1}^{n} m_{\theta}(X_i)$ where $X_1,X_2 \cdots \overset{i.i.d}{\sim} P$ satisfies $\hat{\theta}_{n} \stackrel{\mathrm{P}}{\rightarrow} \theta_{0}$, then 
$$\sqrt{n}\left(\hat{\theta}_{n}-\theta_{0}\right)=-V_{\theta_{0}}^{-1} \frac{1}{\sqrt{n}} \sum_{i=1}^{n} \dot{m}_{\theta_{0}}\left(X_{i}\right)+o_{P}(1).$$
In particular, the sequence $\sqrt{n}\left(\hat{\theta}_{n}-\theta_{0}\right)$ is asymptotically normal with mean zero and covariance matrix $V_{\theta_{0}}^{-1} P \dot{m}_{\theta_{0}} \dot{m}_{\theta_{0}}^{T} V_{\theta_{0}}^{-1}$.
\end{thm*}

\begin{restatable}{lemma}{MLEoptimal}
\label{MLEoptimal}
Suppose $\alpha > 1$ and fix any $\theta_0 \in \mathbb{R}$. If $X_1, X_2\cdots$ are drawn i.i.d from $F_{\theta_0}$ and for any $n\in \mathbb{N}$, $\hat{\theta}_n = \arg\max_{\theta} \sum_{i=1}^{n} \log f_{\theta}(X_i)$, then there exists some $\sigma>0$ such that
$$\sqrt{n}\left(\hat{\theta}_{n}-\theta_0\right) \stackrel{d}{\longrightarrow} N\left(0, \sigma^2\right).$$
\end{restatable}

\begin{proof}
Throughout the proof, let $\mathbb{E}[g(X)]=\intop g(x) f_{\theta_0}(x)dx$ denote the expectation under the true parameter $\theta_0$ and put $\mathbb{E}_{n}[g(X)]=\frac{1}{n}\sum_{i=1}^{n} g(X_i)$ denote the expectation under the empirical distribution. Define $m_{\theta}(x)= \log f_{\theta}(x)-C_0 = - | x-\theta |^{\alpha}$, $M_{n}(\theta) = \E_n\left[ m_{\theta}(X) \right]$ and $M(\theta)=\E[ m_{\theta}(X)]$. (In the notation of Van der Vaart, we have $M_{n}(\theta) = Pm_{\theta}$.) The maximum likelihood estimator is the unique maximizer of $M_n(\cdot)$, the true parameter is the unique maximizer of $M(\cdot)$ and by the law of the large number $M_{n}(\theta) \to M(\theta)$ almost surely for each fixed $\theta$. 

It is clear that $m_{\theta}(\cdot)$ is infitely differentiable except at the singe point $\theta$. We check the remaining conditions needed 

\begin{itemize}
    \item Step 1: $\hat{\theta}_n  \stackrel{\mathrm{P}}{\rightarrow} \theta_0$ .\\
    First, we show that there is a compact interval $\Theta$ containing $\theta_0$ such that $\Prob(\hat{\theta}_n \in \Theta)\to 1.$ Let $\bar{X}_n = \frac{1}{n}\sum_{1}^{n} X_i$ denote the empirical mean. By concavity, we have
    \[ 
    M_{n}(\theta) =  \E_n\left[m_{\theta}(X) \right]   \leq m_{\theta}(\bar{X}_n ) = - \left|  \bar{X}_n - \theta\right|^{\alpha }.
    \] 
    For given $\epsilon>0$ we can take (random) $N$ sufficiently large such that for every $n\geq N$,  $|\bar{X}_n -\theta_0| \leq \epsilon$ and $M_{n}(\theta_0)\geq M(\theta_0) - \epsilon$. Then, it is clear that there exists a bounded interval $\Theta$ containing $\theta_0$ (whose width is dependent on $\epsilon$) such that $M_{n}(\theta) < M_{n}(\theta_0)$ for every $\theta \notin \Theta$.  

    Take $\bar{\theta}_n = \arg\max_{\theta \in \Theta} M_{n}(\theta)$ to be the maximum likelihood estimator restricted to this compact set. By the argument above, $\hat{\theta}_n-\bar{\theta}_n \to 0$ almost surely. Hence, it suffices to verify the consistency of the $\bar{\theta}_n$. This follows by Theorem 5.7 of \cite{van2000asymptotic}. To apply this theorem, we use that (i) The collection $(M_{n}(\theta))_{\theta \in \Theta}$ obeys a uniform law of large numbers, meaning $\sup_{\theta \in \Theta} | M_{n}(\theta) - M(\theta) | \overset{p}{\to} 0$ and (ii) the population criterion $M(\cdot)$ is strictly concave, so $ \sup_{\theta: | \theta - \theta_0| \geq \epsilon} M(\theta) < M(\theta_0).$ 

    \item Step 2: The local Liphshitz condition \eqref{eq:vdv1} is satisfied.\\
      For any $\theta_1,\, \theta_2 \in [\theta_0-h, \theta_0+h],$ we have the following condition holds:
    $$\left|m_{\theta_{1}}(x)-m_{\theta_{2}}(x)\right|=\| x-\theta_{1}|^{\alpha}-| x-\theta_{2}|^{\alpha}| \leq \alpha \max\{|x-\theta_0-h|^{\alpha}, |x-\theta_0+h|^{\alpha}\}\left|\theta_{1}-\theta_{2}\right|.$$
    \item Step 3: The Taylor expansion in \eqref{eq:vdv2} applies.\\
    It is clear that $M'(\theta_0) =0$, since $\theta_0$ is the maximizer. We will show the second derivative $M''(\cdot)$ exists and is continuous in a neighborhood of $\theta_0$. To see why this suffices, observe that by the mean-value form of Taylor's theorem, if $\theta>\theta_0$ is another point in the neighborhood of $\theta_0$ then there exists $\tilde{\theta} \in [\theta_0, \theta]$ such that 
    \begin{align*}
    M(\theta) = M(\theta_0) + \frac{1}{2} M''(\tilde{\theta}) (\theta-\theta_0)^2 
    &=  M(\theta_0) + \frac{1}{2} M''(\theta_0) (\theta-\theta_0)^2 + o\left( \| \theta-\theta_0 \| \right). 
    \end{align*}

    We now show $M''(\theta)$ exits and is continuous.  To do this, we argue that $m_{\theta}$ is twice differentiable almost everywhere and justify the change of limit and integral so that $M''(\theta)=\E\left[ \frac{\partial^2 m_\theta(x)}{\partial \theta^2}\right].$ By symmetry, we only need to consider $\theta > \theta_0$. 
    
    Notice that for any $x$, $\frac{\partial m_\theta(x)}{\partial \theta}$ and $\frac{\partial^2 m_\theta(x)}{\partial \theta^2}$ exist and are continuous except for $\theta = x$. If $\theta<x$, then $\frac{\partial m_\theta(x)}{\partial \theta} = \alpha (x-\theta)^{\alpha-1},\,\, \frac{\partial^2 m_\theta(x)}{\partial \theta^2} = -\alpha (\alpha-1) (x-\theta)^{\alpha-2}$. If $\theta>x$, then $\frac{\partial m_\theta(x)}{\partial \theta} = -\alpha (\theta-x)^{\alpha-1},\,\, \frac{\partial^2 m_\theta(x)}{\partial \theta^2} = -\alpha (\alpha-1) (\theta-x)^{\alpha-2}$. 
    
    Meanwhile notice that $\frac{\partial^2 m_\theta(x)}{\partial \theta^2}$ is integrable:
    $$\begin{aligned}
    \mathbb{E}\left[\left|\frac{\partial^2 m_\theta(x)}{\partial \theta^2}\right|\right]  = \int_{-\infty}^{+\infty}  \alpha (\alpha-1) |x-\theta|^{\alpha-2} e^{-|x-\theta_0|^\alpha} dx < +\infty,
    \end{aligned}$$
    where the fact that integral is finite depends on the fact that $\alpha>1$ and can be verified carefully by integration by parts (similar to the argument below).  
    
    Using the leibniz rule,  
    $$\begin{aligned}
    \frac{\partial^2 M(\theta)}{\partial \theta^2} = \mathbb{E}\left[ \frac{\partial^2 m_\theta(X)}{\partial \theta^2}\right] &= \mathbb{E}\left[ \frac{\partial^2 m_\theta(X)}{\partial X^2}\right] 
    \\ & = - \int_{\theta}^{+\infty} \alpha^2 (x-\theta)^{\alpha-1}(x-\theta_0)^{\alpha-1}e^{-(x-\theta_0)^\alpha} dx 
    \\&+ \int_{\theta_0}^{\theta} \alpha^2 (\theta-x)^{\alpha-1}(x-\theta_0)^{\alpha-1}e^{-(x-\theta_0)^\alpha} dx 
    \\&- \int_{-\infty}^{\theta_0} \alpha^2 (\theta-x)^{\alpha-1}(\theta_0-x)^{\alpha-1}e^{-(\theta_0-x)^\alpha} dx
    \end{aligned}.
    $$
    The last equality is due to integration by parts. We can see that $ \frac{\partial^2 M(\theta)}{\partial \theta^2}$ is bounded and continuous. Meanwhile, notice that $ \frac{\partial^2 M(\theta)}{\partial \theta^2}|_{\theta = \theta_0} = - 2 \int_{\theta_0}^{+\infty} \alpha^2 (x-\theta_0)^{2\alpha-2}e^{-(x-\theta_0)^\alpha} dx <0$, thus it is nonzero, which implies the existence of finite variance $\sigma^2$.
\end{itemize}
\end{proof}

Now we can prove Theorem \ref{estimationthm}:
\estimationthm*

\begin{proof}
We condition on the first $N$ samples which is used for estimating $\theta$ and take expectation. According to the definition of plug-in policy $\tau_N$, we have
$$ \mathbb{E}_{\theta}[\gamma^{\tau_N}X_{\tau_N}] = \mathbb{E}_{\theta}\mathbb{E}_{\theta}\left[\gamma^{\tau_N}X_{\tau_N }| X_1,\,X_2,\,...\,,X_N\right] = \gamma^N \mathbb{E}_{\theta}\mathbb{E}_{\theta}\left[\gamma^{\tau^*({\theta+\hat{\epsilon}})}X_{\tau^*({\theta+\hat{\epsilon}})}| X_1,\,X_2,\,...\,,X_N\right],$$
where $\hat{\epsilon}$ is the error of plug-in estimator for $\theta$ based on the first $N$ samples:
$$\hat{\epsilon} = \hat{\theta} - \theta.$$
Thus we have
\begin{equation}
\label{conditional}
    \Rc(F_{\theta}, \gamma, \tau_N) = 1 - \gamma^N (1-\mathbb{E}_{\theta}\left[\Rc(F_{\theta}, \gamma, \tau^*({\theta+\hat{\epsilon}}))| X_1,\,X_2,\,...\,,X_N\right]).
\end{equation}
Notice that
\begin{equation}
\label{decomposition_regret}
\begin{aligned}
    \mathbb{E}_{\theta}\left[\Rc(F_{\theta}, \gamma, \tau^*({\theta+\hat{\epsilon}}))| X_1,\,X_2,\,...\,,X_N\right] &= \mathbb{P}(\hat{\epsilon} >0) \mathbb{E}_{\theta}\left[\Rc(F_{\theta}, \gamma, \tau^*({\theta+\hat{\epsilon}}))| \hat{\epsilon}>0\right] 
    \\& + \mathbb{P}(\hat{\epsilon} \leq 0) \mathbb{E}_{\theta}\left[\Rc(F_{\theta}, \gamma, \tau^*({\theta+\hat{\epsilon}}))| \hat{\epsilon} \leq 0\right]
\end{aligned}
\end{equation}
The first term and second term in the right-hand side of equation \eqref{decomposition_regret} correspond to the contribution of overestimation and underestimation to the expected regret, respectively.

For the overestimation term, we continue to decompose it. Based on the sample size, we use different formulas.
\begin{enumerate}
    \item $\alpha >1$ and $N = o\left(\left(\frac{\left(\log{\frac{1}{1-\gamma}}\right)^{1-\frac{1}{\alpha}}}{\log\log{\frac{1}{1-\gamma}}}\right)^2\right)$
    
    In this case, we decompose the overestimation term as follows:
    \begin{equation}
    \label{epsilon_1decompose}
    \begin{aligned}
\mathbb{P}(\hat{\epsilon} > 0) \mathbb{E}_{\theta}\left[\Rc(F_{\theta}, \gamma, \tau^*({\theta+\hat{\epsilon}}))| \hat{\epsilon} > 0\right] &= \mathbb{P}(\hat{\epsilon} \geq \epsilon_1(\gamma)) \mathbb{E}_{\theta}\left[\Rc(F_{\theta}, \gamma, \tau^*({\theta+\hat{\epsilon}}))| \hat{\epsilon} \geq \epsilon_1(\gamma)\right]
\\ + \mathbb{P}(0 < \hat{\epsilon} < \epsilon_1(\gamma)) & \mathbb{E}_{\theta}\left[\Rc(F_{\theta}, \gamma, \tau^*({\theta+\hat{\epsilon}}))| 0 < \hat{\epsilon} < \epsilon_1(\gamma)\right]
    \end{aligned}
    \end{equation}
    where
    $$\epsilon_1(\gamma) :=  \frac{2}{\alpha}\left(\log{\frac{1}{1-\gamma}}\right)^{\frac{1}{\alpha}-1} \log\log{\frac{1}{1-\gamma}}.$$
    Note that $N(\gamma) \rightarrow +\infty$ as $\gamma \rightarrow 1$; by Lemma 7, $\sqrt{N}\hat{\epsilon} \stackrel{d}{\longrightarrow} N\left(0, \sigma^2\right)$, where $\sigma^2 >0$. Also $\epsilon_1(\gamma) = o\left(\frac{1}{\sqrt{N(\gamma)}}\right)$. Thus as $\gamma \rightarrow 1$, we have $\mathbb{P}(\hat{\epsilon} \geq \epsilon_1(\gamma)) \rightarrow \frac{1}{2}$ and $\mathbb{P}(0 < \hat{\epsilon} < \epsilon_1(\gamma)) \rightarrow 0$. Meanwhile, by Theorem \ref{thmpositive} and bounded convergence theorem, we know that $\mathbb{E}_{\theta}\left[\Rc(F_{\theta}, \gamma, \tau^*({\theta+\hat{\epsilon}}))| \hat{\epsilon} \geq \epsilon_1(\gamma)\right] \rightarrow 1$. Therefore as $\gamma \rightarrow 1$, it holds that
    $$\mathbb{P}(\hat{\epsilon} > 0) \mathbb{E}_{\theta}\left[\Rc(F_{\theta}, \gamma, \tau^*({\theta+\hat{\epsilon}}))| \hat{\epsilon} > 0\right] \longrightarrow \frac{1}{2}.$$
    
    \item $\alpha >1$ and $N = \omega\left(\left(\frac{\left(\log{\frac{1}{1-\gamma}}\right)^{1-\frac{1}{\alpha}}}{\log\log{\frac{1}{1-\gamma}}}\right)^2\right)$
    
    In this case, we decompose the overestimation term as follows:
    \begin{equation}
    \label{epsilon_2decompose}
    \begin{aligned}
\mathbb{P}(\hat{\epsilon} > 0) \mathbb{E}_{\theta}\left[\Rc(F_{\theta}, \gamma, \tau^*({\theta+\hat{\epsilon}}))| \hat{\epsilon} > 0\right] &= \mathbb{P}(\hat{\epsilon} \geq \epsilon_2(\gamma)) \mathbb{E}_{\theta}\left[\Rc(F_{\theta}, \gamma, \tau^*({\theta+\hat{\epsilon}}))| \hat{\epsilon} \geq \epsilon_2(\gamma)\right]
\\ + \mathbb{P}(0 < \hat{\epsilon} < \epsilon_2(\gamma)) & \mathbb{E}_{\theta}\left[\Rc(F_{\theta}, \gamma, \tau^*({\theta+\hat{\epsilon}}))| 0 < \hat{\epsilon} < \epsilon_2(\gamma)\right]
    \end{aligned}
    \end{equation}
    where
    $$\epsilon_2(\gamma) :=  \frac{1}{2\alpha}\left(\log{\frac{1}{1-\gamma}}\right)^{\frac{1}{\alpha}-1} \log\log{\frac{1}{1-\gamma}}.$$
    Note that $\epsilon_2(\gamma) = \omega\left(\frac{1}{\sqrt{N(\gamma)}}\right)$, thus as $\gamma \rightarrow 1$, we have $\mathbb{P}(\hat{\epsilon} \geq \epsilon_2(\gamma)) \rightarrow 0$ and $\mathbb{P}(0 < \hat{\epsilon} < \epsilon_2(\gamma)) \rightarrow \frac{1}{2}$. Notice that $\mathbb{E}_{\theta}\left[\Rc(F_{\theta}, \gamma, \tau^*({\theta+\hat{\epsilon}}))| \hat{\epsilon} \geq \epsilon_2(\gamma)\right]$ is bounded by $1$. Meanwhile, by Theorem \ref{thmpositive} and bounded convergence theorem, we know that $\mathbb{E}_{\theta}\left[\Rc(F_{\theta}, \gamma, \tau^*({\theta+\hat{\epsilon}}))| 0 < \hat{\epsilon} < \epsilon_2(\gamma)\right] \rightarrow 0$. Therefore as $\gamma \rightarrow 1$, it holds that
    $$\mathbb{P}(\hat{\epsilon} > 0) \mathbb{E}_{\theta}\left[\Rc(F_{\theta}, \gamma, \tau^*({\theta+\hat{\epsilon}}))| \hat{\epsilon} > 0\right] \longrightarrow 0.$$

\end{enumerate}

For the underestimation term in the right-hand side of equation \eqref{decomposition_regret}, we can decompose it in a unified way regardless of different cases:
$$\begin{aligned}
\mathbb{P}(\hat{\epsilon} \leq 0) \mathbb{E}_{\theta}\left[\Rc(F_{\theta}, \gamma, \tau^*({\theta+\hat{\epsilon}}))| \hat{\epsilon} \leq 0\right] &= \mathbb{P}(\hat{\epsilon} \leq -\epsilon_3(\gamma)) \mathbb{E}_{\theta}\left[\Rc(F_{\theta}, \gamma, \tau^*({\theta+\hat{\epsilon}}))| \hat{\epsilon} \leq -\epsilon_3(\gamma)\right]
\\ + \mathbb{P}(\epsilon_3(\gamma) < \hat{\epsilon} \leq 0) & \mathbb{E}_{\theta}\left[\Rc(F_{\theta}, \gamma, \tau^*({\theta+\hat{\epsilon}}))| \epsilon_3(\gamma) < \hat{\epsilon} \leq 0\right]
\end{aligned}$$
where
$$\epsilon_3(\gamma) := \left(\log{\frac{1}{1-\gamma}}\right)^{\frac{1}{2\alpha}}.$$
As $\gamma \rightarrow 1$, we have $\epsilon_3(\gamma) \rightarrow +\infty$. By the similar argument as in the case 2 for computing overestimation term, we obtain the result that as $\gamma \rightarrow 1$, it holds that
\begin{equation}
\label{underestimation_ineq}
    \mathbb{P}(\hat{\epsilon} \leq 0) \mathbb{E}_{\theta}\left[\Rc(F_{\theta}, \gamma, \tau^*({\theta+\hat{\epsilon}}))| \hat{\epsilon} \leq 0\right] \longrightarrow 0.
\end{equation}
    

Notice that if $N = o\left(\frac{1}{1-\gamma}\right)$, then $\gamma ^N \rightarrow 1$ as $\gamma \rightarrow 1.$ Combining the computation for different cases and equation \eqref{conditional}, we prove the result stated in Theorem \ref{estimationthm}.
\end{proof}

\bigskip
We recall Theorem \ref{estimationthm_onesample}:
\onesample*

\begin{proof}
The proof of Theorem \ref{estimationthm_onesample} shares exactly the same spirit with the proof of Theorem \ref{estimationthm}. Because there is only single sample, we have $\hat{\theta} = X_1$, thus 
$$\hat{\epsilon} = X_1 - \theta.$$
We still consider the equation \eqref{decomposition_regret} and decompose the overestimation term as equation \eqref{epsilon_2decompose}. Note that $\alpha \leq 1$ implies $\epsilon_2(\gamma) \rightarrow +\infty$ as $\gamma \rightarrow 1$, thus we have
$$\lim_{\gamma \rightarrow 1}\mathbb{P}\{|\hat{\epsilon}| \geq \epsilon_2(\gamma)\} =  \lim_{\epsilon \rightarrow +\infty}\mathbb{P}\{|X_1-\theta| > \epsilon\} = 0.$$
By the similar argument as in the case 2 for computing overestimation term, we obtain the same result that as $\gamma \rightarrow 1$, it holds that
$$\mathbb{P}(\hat{\epsilon} > 0) \mathbb{E}_{\theta}\left[\Rc(F_{\theta}, \gamma, \tau^*({\theta+\hat{\epsilon}}))| \hat{\epsilon} > 0\right] \longrightarrow 0.$$
Combining with the same underestimation result \eqref{underestimation_ineq}, we finish the proof of Theorem \ref{estimationthm_onesample}.
\end{proof}

\section{Proof of Proposition \ref{prop:stopping_prob}}
\label{proofprop2}
In this section, we prove the Proposition \ref{prop:stopping_prob}. We use  ground truth stopping probability $P_{0}(S^*(F_0,\gamma))$ as an intermediate, and transfer the comparison of $P_{0}(S^*(F_\gamma,\gamma))$ with $1-\gamma$ to that of $P_{0}(S^*(F_\gamma,\gamma))$ with $P_{0}(S^*(F_0,\gamma))$. Since we from now on we only consider the exponential-decay distribution $F_{\theta}$ specified in \eqref{distribution}, we slightly change the notation without ambiguity.

We first introduce a lemma focusing on the first-order approximation of the tailed probability ratio 
$\frac{P_{\theta_0}(S^*(\theta_0,\gamma))}{P_{\theta_0}(S^*(\theta_0+\epsilon(\gamma),\gamma))}.$
To simplify our notation, we denote $S^*(\theta_0,\gamma)$ as $S^*_0(\gamma)$.

\begin{restatable}{lemma}{lemma6.3}
\label{lemma6.3}
Suppose $\Delta(\gamma)$ is a non-negative function satisfying $\Delta(\gamma) = o\left(\left(\log{\frac{1}{1-\gamma}}\right)^{\frac{1}{\alpha}}\right)$, then for any constant $c_0,\, c_0>0$, there exists $\gamma_0$ such that for every $\gamma \geq \gamma_0$, it holds that
\begin{equation}
\label{6.3approx}
    \exp{\left((1-c_0) h_{\theta_0}(S^*_0(\gamma))\cdot \Delta(\gamma)\right)}\leq \frac{P_{\theta_0}(S^*_0(\gamma))}{P_{\theta_0}(S^*_0(\gamma)+\Delta(\gamma))}
    \leq \exp{\left((1+c_0) h_{\theta_0}(S^*_0(\gamma))\cdot \Delta(\gamma)\right)}.
\end{equation}
\end{restatable}

\begin{proof}
Without loss of generality we assume $S^*_0(\gamma) > \theta_0+1$. (By Lemma \ref{Lemma5.1}, we know that there exists $\gamma_0\prime$ such that for every $\gamma \geq \gamma_0\prime$, $S^*_0(\gamma) > \theta_0+1$.) Notice that $$ \frac{P_{\theta_0}(S^*_0(\gamma))}{P_{\theta_0}(S^*_0(\gamma)+\Delta(\gamma))} = \exp{(\log{P_{\theta_0}(S^*_0(\gamma))} - \log{P_{\theta_0}(S^*_0(\gamma)+\Delta(\gamma))})},$$ 
by Mean Value Theorem,
\begin{equation}
\label{6.3taylor}
    \log{P_{\theta_0}(S^*_0(\gamma))} - \log{P_{\theta_0}(S^*_2(\gamma))} = -\frac{\partial \log{P_{\theta_0}(x)}}{\partial x} (S^*_0(\gamma)) \cdot \Delta(\gamma) -\frac{1}{2}\cdot \frac{\partial^2 \log{P_{\theta_0}(x)}}{\partial x^2}(\Tilde{S}) \Delta(\gamma)^2,
\end{equation}
where $\Tilde{S} \in [S^*_0(\gamma), S^*_0(\gamma)+\Delta(\gamma)].$

\medskip
By Lemma \ref{qiudaobudengshi} and Lemma \ref{uniform g and h}, we have 
$$\frac{\partial \log{P_{\theta_0}(x)}}{\partial x}(S_0^*(\gamma)) = -h_{\theta_0}(S_0^*(\gamma)) \sim -\alpha (S_0^*(\gamma)-\theta_0)^{\alpha-1},$$
thus there exists $\gamma_1$ such that for every $\gamma \geq \gamma_1$, it holds that
\begin{equation}
\label{yijiedaobound}
    \left|\frac{\partial \log{P_{\theta_0}(x)}}{\partial x}(S_0^*(\gamma))\right| \geq \frac{1}{2}\cdot \alpha (S_0^*(\gamma)-\theta_0)^{\alpha-1}.
\end{equation}

Next we compute a uniform bound on the second derivative term.
\begin{itemize}
    \item If $\alpha > 1$
    
    By Lemma \ref{qiudaobudengshi}, for $\forall S\in [S^*_0(\gamma), S^*_0(\gamma)+\Delta(\gamma)]$, it holds that 
 \begin{equation}
 \label{secondderivative_order}
    \left|\frac{\partial^2 \log{P_{\theta_0}(x)}}{\partial x^2}(S)\right| \leq \frac{\alpha^3 (\alpha-1) (S-\theta_0)^{3\alpha-2}}{\left(\alpha(S-\theta_0)^{\alpha} -(\alpha-1) \right)^2}.
 \end{equation}
    \begin{itemize}
        \item  If $\alpha \geq 2$, we can pick up $\gamma^{\prime}_2$ such that for every $S > S^*_0(\gamma^{\prime}_2)$, the function in the right-hand side of equation \eqref{secondderivative_order} is monotonically increasing. Thus for every $\gamma > \gamma^{\prime}_2$, it holds that
$$\begin{aligned}
\max_{S\in [S^*_0(\gamma), S^*_0(\gamma)+\Delta(\gamma)]} \left|\frac{\partial^2 \log{P_{\theta_0}(x)}}{\partial x^2}(S)\right| &\leq \frac{\alpha^3 |\alpha-1| (S^*_0(\gamma)+\Delta(\gamma)-\theta_0)^{3\alpha-2}}{\left(\alpha(S^*_0(\gamma)+\Delta(\gamma)-\theta_0)^{\alpha} -(\alpha-1) \right)^2}
\\ &\sim \alpha |\alpha-1|(S^*_0(\gamma)+\Delta(\gamma)-\theta_0)^{\alpha-2}.
\end{aligned}$$
Thus there exists $\gamma_2 \geq \gamma_2^{\prime}$ such that for every $\gamma \geq \gamma_2$, $S^*_0(\gamma)+\Delta(\gamma) \geq S^*_0(\gamma) \geq S^*_0(\gamma_2)$, and it holds that
\begin{equation}
\label{6.3.erjiedaobound1}
    \max_{S\in [S^*_0(\gamma), S^*_0(\gamma)+\Delta(\gamma)]} \left|\frac{\partial^2 \log{P_{\theta_0}(x)}}{\partial x^2}(S)\right| \leq 2\alpha |\alpha-1|(S^*_0(\gamma)+\Delta(\gamma)-\theta_0)^{\alpha-2}.
\end{equation}
     \item  If $\alpha \in \left(1,2\right)$, we can pick up $\gamma^{\prime}_2$ such that for every $S > S^*_0(\gamma^{\prime}_2)$, the function in the right-hand side of equation \eqref{secondderivative_order} is monotonically decreasing. Thus for every $\gamma > \gamma^{\prime}_2$, it holds that
$$\begin{aligned}
\max_{S\in [S^*_0(\gamma), S^*_0(\gamma)+\Delta(\gamma)]} \left|\frac{\partial^2 \log{P_{\theta_0}(x)}}{\partial x^2}(S)\right| &\leq \frac{\alpha^3 |\alpha-1| (S^*_0(\gamma)-\theta_0)^{3\alpha-2}}{\left(\alpha(S^*_0(\gamma)-\theta_0)^{\alpha} -(\alpha-1) \right)^2}
\\ &\sim \alpha |\alpha-1|(S^*_0(\gamma)-\theta_0)^{\alpha-2}.
\end{aligned}$$
Thus there exists $\gamma_2 \geq \gamma_2^{\prime}$ such that for every $\gamma \geq \gamma_2$, $S^*_0(\gamma) \geq S^*_0(\gamma_2)$, and it holds that
\begin{equation}
\label{6.3.erjiedaobound}
    \max_{S\in [S^*_0(\gamma), S^*_0(\gamma)+\Delta(\gamma)]} \left|\frac{\partial^2 \log{P_{\theta_0}(x)}}{\partial x^2}(S)\right| \leq 2\alpha |\alpha-1|(S^*_0(\gamma)-\theta_0)^{\alpha-2}.
\end{equation}
    \end{itemize}
    \item If $\alpha \in \left[\frac{1}{2},1\right]$
    
    By Lemma \ref{qiudaobudengshi}, for $\forall S\in [S^*_0(\gamma), S^*_0(\gamma)+\Delta(\gamma)]$, it holds that 
$$\begin{aligned}
\left|\frac{\partial^2 \log{P_{\theta}(x)}}{\partial x^2}(S)\right| &\leq \alpha|\alpha-1|(S-\theta)^{\alpha-2}
\\ &\sim \alpha|\alpha-1|(S^*_0(\gamma)-\theta)^{\alpha-2}.
\end{aligned}$$
 Thus there exists $\gamma_2$ such that for every $\gamma \geq \gamma_2$, $S^*_0(\gamma) \geq S^*_0(\gamma_2)$, \eqref{6.3.erjiedaobound} holds.
\end{itemize}
Therefore, for any $\alpha \geq \frac{1}{2}$, there exists $\gamma_2$ such that for any $\gamma \geq \gamma_2$, it holds that
\begin{equation}
\label{erjiedaobound}
    \max_{S\in [S^*_0(\gamma), S^*_0(\gamma)+\Delta(\gamma)]} \left|\frac{\partial^2 \log{P_{\theta_0}(x)}}{\partial x^2}(S)\right| \leq 2\alpha |\alpha-1|\max{\{(S^*_0(\gamma)-\theta_0)^{\alpha-2}, (S^*_0(\gamma)+\Delta(\gamma)-\theta_0)^{\alpha-2}\}}.
\end{equation}
As $\gamma \rightarrow 1$, as long as $\Delta(\gamma) = S^*_0(\gamma)+\Delta(\gamma) - S^*_0(\gamma) = o\left(\left(\log{\frac{1}{1-\gamma}}\right)^{\frac{1}{\alpha}}\right)$, according to Lemma \ref{Lemma5.1}, there exists $\gamma_3$ such that for every $\gamma \geq \gamma_3$, it holds that 
$$\Delta(\gamma) \leq \frac{1}{2}\left(\log{\frac{1}{1-\gamma}}\right)^{\frac{1}{\alpha}}\leq S^*_0(\gamma) - \theta_0,$$
hence $S^*_0(\gamma)-\theta_0 \leq S^*_0(\gamma)+\Delta(\gamma)-\theta_0 \leq 2(S^*_0(\gamma)-\theta_0)$ holds as $\gamma \geq \gamma_3$.
Combining with \eqref{yijiedaobound} and \eqref{erjiedaobound}, we have as $\gamma \geq \max{\{\gamma_1,\gamma_2,\gamma_3\}}$, it holds that
$$\frac{\max_{S\in [S^*_0(\gamma), S^*_0(\gamma)+\Delta(\gamma)]} \left|\frac{\partial^2 \log{P_{\theta_0}(x)}}{\partial x^2}(S)\right|}{\left|\frac{\partial \log{P_{\theta_0}(x)}}{\partial x}(S_0^*(\gamma))\right|} \leq \frac{\max{\{2,2^{\alpha-1}\}}\alpha |\alpha-1|(S^*_0(\gamma)-\theta_0)^{\alpha-2}}{\frac{1}{2}\cdot \alpha (S_0^*(\gamma)-\theta_0)^{\alpha-1}} \leq \frac{2^{\alpha+2}|\alpha-1|}{S_0^*(\gamma)-\theta_0}.$$
Notice that $S_0^*(\gamma)-\theta_0 \sim \left(\log{\frac{1}{1-\gamma}}\right)^{\frac{1}{\alpha}}$ while $\Delta(\gamma) = o\left(\left(\log{\frac{1}{1-\gamma}}\right)^{\frac{1}{\alpha}}\right)$, thus for any $c_0,\, c_0>0$, there exists $\gamma_0,\, \gamma_0 \geq \max{\{\gamma_1,\gamma_2,\gamma_3\}}$ such that
$$\frac{2^{\alpha}|\alpha-1|\cdot \Delta(\gamma)}{S_0^*(\gamma)-\theta_0} \leq c_0.$$

Therefore for every $\gamma \geq \gamma_0$, it holds that
$$\max_{S\in [S^*_0(\gamma), S^*_0(\gamma)+\Delta(\gamma)]} \left|\frac{\partial^2 \log{P_{\theta_0}(x)}}{\partial x^2}(S)\right|\Delta(\gamma)^2 < c_0\cdot \left|\frac{\partial \log{P_{\theta_0}(x)}}{\partial x}(S_0^*(\gamma))\right| \Delta(\gamma),$$
which combined with Lemma \ref{qiudaobudengshi} and the Taylor Expansion equation \eqref{6.3taylor} leads to the equation \eqref{6.3approx}.
\end{proof}

Now we are well-prepared to prove Proposition \ref{prop:stopping_prob}.
\propOnTimeStopping*

\begin{proof}[Proof of first claim]
Without loss of generality, we assume that $\epsilon(\gamma)$ satisfies \eqref{thetabound}, i.e., $\epsilon(\gamma) = o\left(\left(\log{\frac{1}{1-\gamma}}\right)^{\frac{1}{\alpha}}\right)$. By Corollary \ref{coro5.6}, for any constant $c_1>0$, there exists $\gamma_0$ such that for every $\gamma \geq \gamma_0$, it holds that
$$\Delta(\gamma) := S^*(\theta_0 + \epsilon, \gamma) - S_0^*(\gamma) \geq \left(1+\frac{c_1}{2}\right)\cdot \frac{\log{\log{\frac{1}{1-\gamma}}}}{\alpha \left(\log{\frac{1}{1-\gamma}}\right)^{1-\frac{1}{\alpha}}}$$
and
$$\Delta(\gamma) \leq \frac{3}{2}\epsilon(\gamma).$$

Notice that by Lemma \ref{uniform g and h}, Lemma \ref{Lemma5.1} and Corollary \ref{proposition5.2}, it holds that
\begin{equation}
\label{SgsASYMO}
    \frac{\log{(S^*_0(\gamma)\cdot g_{\theta_0}(S^*_0(\gamma))})}{h_{\theta_0}(S^*_0(\gamma))} \sim \frac{\log{\log{\frac{1}{1-\gamma}}}}{\alpha \left(\log{\frac{1}{1-\gamma}}\right)^{1-\frac{1}{\alpha}}},
\end{equation}
thus there exists $\gamma_0^{\prime} \geq \gamma_0$ such that for every $\gamma \geq \gamma_0^{\prime}$, it holds that
\begin{equation}
\label{deltalarge}
    \left(1+\frac{c_1}{4}\right)\cdot \frac{\log{(S^*_0(\gamma)\cdot g_{\theta_0}(S^*_0(\gamma))})}{h_{\theta_0}(S^*_0(\gamma))} \leq \Delta(\gamma) \leq \frac{3}{2} \epsilon(\gamma).
\end{equation}


By first-order approximation Lemma \ref{lemma6.3}, there exists $\gamma_1^{\prime} \geq \gamma_0^{\prime}$ such that for every $\gamma \geq \gamma_1^{\prime}$, it holds that
$$\frac{P_{\theta_0}(S_0^*(\gamma))}{P_{\theta_0}(S^*(\theta_0+\epsilon(\gamma),\gamma))} \geq \left(S^*_0(\gamma)\cdot g_{\theta_0}(S^*_0(\gamma))\right)^{1+\frac{c_1}{8}},$$

Recall that
$$ P_{\theta_0}(S^*_0(\gamma))= \frac{ S^*_0(\gamma) g_{\theta_0}(S^*_0(\gamma)) \cdot (1-\gamma)}{\gamma}.$$

hence
$$P_{\theta_0}(S^*(\theta_0+\epsilon(\gamma),\gamma)) \leq \frac{1-\gamma}{\gamma} \cdot \left(S^*_0(\gamma)\cdot g_{\theta_0}(S^*_0(\gamma))\right)^{-\frac{c_1}{8}}.$$
By Corollary \ref{proposition5.2}, we know that $S^*_0(\gamma)\cdot g_{\theta_0}(S^*_0(\gamma))\sim \alpha \log{\frac{1}{1-\gamma}}$, thus for any constant $m_0>0$, there exists $\gamma_1 \geq \gamma_1^{\prime}$ such that for every $\gamma \geq \gamma_1$, it holds that $ P_{\theta_0}(S^*(\theta_0+\epsilon(\gamma),\gamma)) \leq m_0(1-\gamma)$.
\end{proof}

\begin{proof}[Proof of second claim]
By Corollary \ref{coro5.6}, for any constant $c_2>0$, there exists $\gamma_0$ such that for every $\gamma \geq \gamma_0$, it holds that
$$\Delta(\gamma) \leq \left(1-\frac{c_2}{2}\right)\cdot \frac{\log{\log{\frac{1}{1-\gamma}}}}{\alpha \left(\log{\frac{1}{1-\gamma}}\right)^{1-\frac{1}{\alpha}}}.$$
By equation \eqref{SgsASYMO}, there exists $\gamma_0^{\prime} \geq \gamma_0$ such that for every $\gamma \geq \gamma_0^{\prime}$, it holds that
\begin{equation}
\label{deltasmall}
    \Delta(\gamma) \leq \left(1-\frac{c_2}{4}\right)\cdot \frac{\log{(S^*_0(\gamma)\cdot g_{\theta_0}(S^*_0(\gamma))})}{h_{\theta_0}(S^*_0(\gamma))}
\end{equation}
By first-order approximation Lemma \ref{lemma6.3}, there exists $\gamma_2^{\prime} \geq \gamma_0^{\prime}$ such that for every $\gamma \geq \gamma_2^{\prime}$, it holds that
$$\frac{P_{\theta_0}(S_0^*(\gamma))}{P_{\theta_0}(S^*(\theta_0+\epsilon(\gamma),\gamma))} \leq \left(S^*_0(\gamma)\cdot g_{\theta_0}(S^*_0(\gamma))\right)^{1-\frac{c_2}{8}},$$
thus
$$P_{\theta_0}(S^*(\theta_0+\epsilon(\gamma),\gamma)) \geq \frac{1-\gamma}{\gamma} \cdot \left(S^*_0(\gamma)\cdot g_{\theta_0}(S^*_0(\gamma))\right)^{\frac{c_2}{8}}.$$
Thus for any $M_0>0$, there exists $\gamma_2 \geq \gamma_2^{\prime}$ such that for every $\gamma \geq \gamma_2$, it holds that $ P_{\theta_0}(S^*(\theta_0+\epsilon(\gamma),\gamma)) \geq M_0(1-\gamma)$.
\end{proof}

\section{Proof of Proposition \ref{phasetransitionlemma}}
\label{proofprop1}
In this section we give the proof of Proposition \ref{phasetransitionlemma}. We first reformulate the Bellman Equation \eqref{threhsoldbellman} through lens of tailed probability and conditional expectation.

\begin{restatable}{lemma}{reformulate}
\label{reformulate}
The optimal threshold $S^*(F,\gamma)$ to the Bellman Equation \eqref{threhsoldbellman} satisfies:
\begin{equation}
\label{reform}
    P_{0}(S^*(F_0,\gamma)) = \frac{1-\gamma}{\gamma}\cdot \frac{S^*(F_0,\gamma)}{\mu_{0}(S^*(F_0,\gamma)) - S^*(F_0,\gamma)}.
\end{equation}
\end{restatable}

\begin{proof}
We reformulate the Bellman Equation \eqref{threhsoldbellman} as
\begin{equation}
\label{newbellman}
    \frac{1}{\gamma}S^{*}(F_0,\gamma)=\frac{1}{1-\gamma} \int_{ S^{*}(F_0,\gamma)}^{+\infty}P_{0}(x) d x.
\end{equation}
Recall $g(S) := \frac{P(S)}{\int_S^{+\infty} P(x) dx}$. By Lemma \ref{lemma3.1}, we obtain the equation \eqref{reform}.
\end{proof}

By Lemma \ref{reformulate}, we immediately obtain the following corollary if the underlying distribution $F_0$ satisfies $\mu_0(S) \sim S$.

\begin{restatable}{coro}{prop4.2}
\label{prop4.2}
If distribution $F_0$ satisfies \eqref{mus-s}, then it holds that 
\begin{equation}
\label{1-gammacondition}
    P_{0}(S^*(F_0, \gamma)) = \omega\left(1-\gamma\right).
\end{equation}
\end{restatable}

Now we are well prepared to prove Proposition \ref{phasetransitionlemma}. We restate it as follows:
\phasetransitionlemma*

\begin{proof}
For any threshold value $S$, the expected discounted reward of the corresponding threshold policy \eqref{policy1.2} performed on the ground truth $F_0$ is
\begin{equation}
\label{performanceformula}
    \frac{\mu_{0}(S)\cdot P_{0}(S)}{1- \gamma (1-P_{0}(S))},
\end{equation}
thus the regret of Bellman policy $\tau^*(F_\gamma)$ is
\begin{equation}
\label{computeformula}
    \Rc(F_0,\gamma,\tau^*(F_\gamma)) = 1 -  \frac{\mu_{0}(S^*(F_\gamma,\gamma))}{\mu_{0}(S^*(F_0,\gamma))} \cdot \frac{P_{0}(S^*(F_\gamma,\gamma))}{P_{0}(S^*(F_0,\gamma))} \cdot \frac{(1-\gamma) + \gamma P_{0}(S^*(F_0,\gamma))}{(1-\gamma)+\gamma P_{0}(S^*(F_\gamma,\gamma))}.
\end{equation}
By \eqref{mus-s} and \eqref{almostsamethreshold}, it holds that
\begin{equation}
\label{re}
    \frac{\mu_{0}(S^*(F_\gamma,\gamma))}{\mu_{0}(S^*(F_0,\gamma))} \longrightarrow 1
\end{equation}
as $\gamma \rightarrow 1$. Meanwhile, by Corollary \ref{prop4.2}, it holds that $P_{0}(S^*(F_0,\gamma))= \omega (1-\gamma)$, thus
$$\frac{P_{0}(S^*(F_\gamma,\gamma))}{P_{0}(S^*(F_0,\gamma))} \cdot \frac{(1-\gamma) + \gamma P_{0}(S^*(F_0,\gamma))}{(1-\gamma)+\gamma P_{0}(S^*(F_\gamma,\gamma))} \sim \frac{P_{0}(S^*(F_\gamma,\gamma))}{(1-\gamma)+\gamma P_{0}(S^*(F_\gamma,\gamma))}.$$
Therefore, if $P_{0}(S^*(F_\gamma,\gamma)) = \omega(1-\gamma)$, then $\lim_{\gamma \rightarrow 1} \Rc(F_0, \gamma, \tau^*(F_\gamma)) = 0$; if $P_{0}(S^*(F_\gamma,\gamma)) = o(1-\gamma)$, then $\lim_{\gamma \rightarrow 1} \Rc(F_0, \gamma, \tau^*(F_\gamma)) = 1$.
\end{proof}

\end{document}